\documentclass{article} 

\usepackage{arxiv} 
\usepackage{enumitem}

\usepackage{times,amsthm,amsfonts,amsmath,amssymb,epsfig,xcolor,float,graphicx,verbatim}
\usepackage{algorithm,algorithmic}

\usepackage{hyperref}
\usepackage{url}
\usepackage{cleveref}

\newcommand{\one}{\mathbf{1}}

\newcommand{\poly}{\mathrm{poly}}

\newcommand{\reals}{\mathbb{R}}
\newcommand{\R}{\mathbb{R}}

\newcommand{\N}{\mathbb{N}}

\newcommand{\zero}{\boldsymbol{0}}

\newcommand{\abs}[1]{\left| #1 \right|}

\newcommand{\vertiii}[1]{{\left\vert\kern-0.25ex\left\vert\kern-0.25ex\left\vert #1 
    \right\vert\kern-0.25ex\right\vert\kern-0.25ex\right\vert}}

\newcommand{\Ecal}{\mathcal{E}}

\newcommand{\bxi}{{\xi}}

\newcommand{\bigO}{\mathcal{O}}

\usepackage{xcolor}

\newcommand{\beq}{\begin{eqnarray*}}
\newcommand{\eeq}{\end{eqnarray*}}
\newcommand{\beqn}{\begin{eqnarray}}
\newcommand{\eeqn}{\end{eqnarray}}

\usepackage{ifmtarg}
\usepackage{xifthen}%
\newcommand{\ent}[1][]{%
\ifthenelse{\isempty{#1}}{%
\mathrm{H}
}{
\mathrm{H}^{(#1)}
}}

\newcommand{\loch}[1][]{%
\ifthenelse{\isempty{#1}}{%
\mathrm{h}
}{
\mathrm{h}^{(#1)}
}}

\newcommand{\Dcal}{\mathcal{D}}

\newcommand{\Ical}{\mathcal{I}}

\newcommand{\Ncal}{\mathcal{N}}
\newcommand{\Rcal}{\mathcal{R}}

\newcommand{\E}{\mathbb{E}}

\newcommand{\argmin}[1]{\underset{#1}{\mathrm{argmin}}}

\newcommand{\Var}{\text{Var}}

\newcommand{\Sphere}{\mathbb{S}}

\newcommand{\be}{\mathbf{e}}
\newcommand{\bx}{\mathbf{x}}

\newcommand{\bw}{\mathbf{w}}

\newcommand{\bb}{\mathbf{b}}
\newcommand{\bu}{\mathbf{u}}
\newcommand{\bv}{\mathbf{v}}
\newcommand{\bz}{\mathbf{z}}
\newcommand{\br}{\mathbf{r}}

\newcommand{\by}{\mathbf{y}}

\newcommand{\blambda}{\boldsymbol{\lambda}}

\newcommand{\Wcal}{\mathcal{W}}

\newcommand{\norm}[1]{\left\|#1\right\|}

\newcommand{\op}{\mathrm{op}}

\renewcommand{\bxi}{\boldsymbol{\xi}}

\newtheorem{theorem}{Theorem}
\newtheorem{proposition}{Proposition}
\newtheorem{lemma}{Lemma}
\newtheorem{corollary}{Corollary}
\newtheorem{definition}{Definition}

\newtheorem{assumption}{Assumption}

\newcommand{\secref}[1]{Sec.~\ref{#1}}

\newcommand{\figref}[1]{Fig.~\ref{#1}}

\newcommand{\thmref}[1]{Thm.~\ref{#1}}

\newcommand{\appref}[1]{Appendix~\ref{#1}}

\title{Benign Overfitting for General Norms\\ 
and Distributions}
\author{
Daniel Barzilai \\
Weizmann Institute of Science \\
\And
Ohad Shamir \\
University of Toronto \& Vector Institute 
\\Weizmann Institute of Science
}

\begin{document}

\maketitle

\vspace{-1em}\begin{abstract}
Understanding why predictors can generalize despite interpolating noisy training data is a central puzzle in machine learning. 
Most work on such ``benign overfitting'' studies minimum-\(\ell_2\)-norm linear regression,
reflecting the inductive bias of gradient descent. 
However, modern optimizers such as Adam and Muon use non-Euclidean update geometries, favoring solutions associated with other norms. Analyzing regression for non-$\ell_2$ norms is substantially more difficult, with known results essentially limited to Gaussian distributions. 
In this paper, we develop a method to analyze benign overfitting in linear regression for general norms \emph{and} general (sub-Gaussian) distributions. 
As a special case, we prove that minimum-\(\ell_p\)-norm interpolation with $p\!>\!1$ can benignly overfit even for non-Gaussian distributions, under suitable conditions. 
Perhaps surprisingly, for the $\ell_1$ norm, benign overfitting \emph{does not} hold in general for well-behaved (but non-Gaussian) distributions, showing that existing positive $\ell_1$ results rely crucially on Gaussianity.
Our proof analyzes the geometry of the dual optimization problem, using concentration and central limit tools to show it is approximately Euclidean in many high-dimensional cases.  
\end{abstract}

\section{Introduction}
\label{sec:introduction}
Classical intuition associates overfitting with poor generalization. However, high-dimensional models can often generalize well even when interpolating the training data (namely, achieving zero training error, which results in overfitting if the Bayes risk is positive). This phenomenon is known as benign overfitting 
\citep{bartlett2020benign, muthukumar2020harmless, belkin2020two,hastie2022surprises, mallinar2022benign, barzilaigeneralization}, and understanding when
it occurs and how it depends on the learning algorithm has become a fundamental question in machine learning theory.

Linear regression provides a canonical setting to study benign overfitting.
Consider an input matrix \(X\in\R^{n\times d}\) with $n$ i.i.d. rows, a
target vector \(\bw^\star\in\R^d\), and noisy labels
\(\by=X\bw^\star+\bxi\), where \(\bxi\in\R^n\) is a vector of independent noise terms. 
In the overparameterized regime where $d>n$, there are typically infinitely many predictors $\bw$ that interpolate this training data, i.e. satisfy $X\bw=\by$. 
Given a norm \(\|\cdot\|\) on \(\R^d\), 
a common choice is the minimum-norm interpolator, given by
\[
    \widehat\bw\in
    \argmin{\bw\in\R^d}\|\bw\|
    \quad\text{subject to}\quad X\bw=\by.
\]
In this setting, benign overfitting corresponds to $\hat{\bw}$ achieving approximately optimal performance (comparable to $\bw^\star$), despite overfitting the noisy labels. 

The large majority of works on linear regression and benign overfitting consider the Euclidean $\ell_2$ norm. 
This choice is natural from an optimization perspective: 
In overparameterized linear regression, gradient descent initialized at zero is well-known to converge to the minimum-\(\ell_2\)-norm interpolator.
However, modern machine learning increasingly uses different optimizers such as Adam \citep{kingma2015adam} and Muon \citep{jordan2024muon}, and it is well known that the choice of optimizer affects the solution reached at convergence \citep{gunasekar2018characterizing}.
For example, in various regimes, Adam and SignSGD are known to select the minimum-\(\ell_\infty\)-norm solution
\citep{zhang2024implicit,bernstein2024old,gronichimplicit}.
This underscores the importance of understanding benign overfitting beyond minimum-\(\ell_2\)-norm interpolation. 

\begin{figure*}[t]
    \centering
    \includegraphics[width=\textwidth]{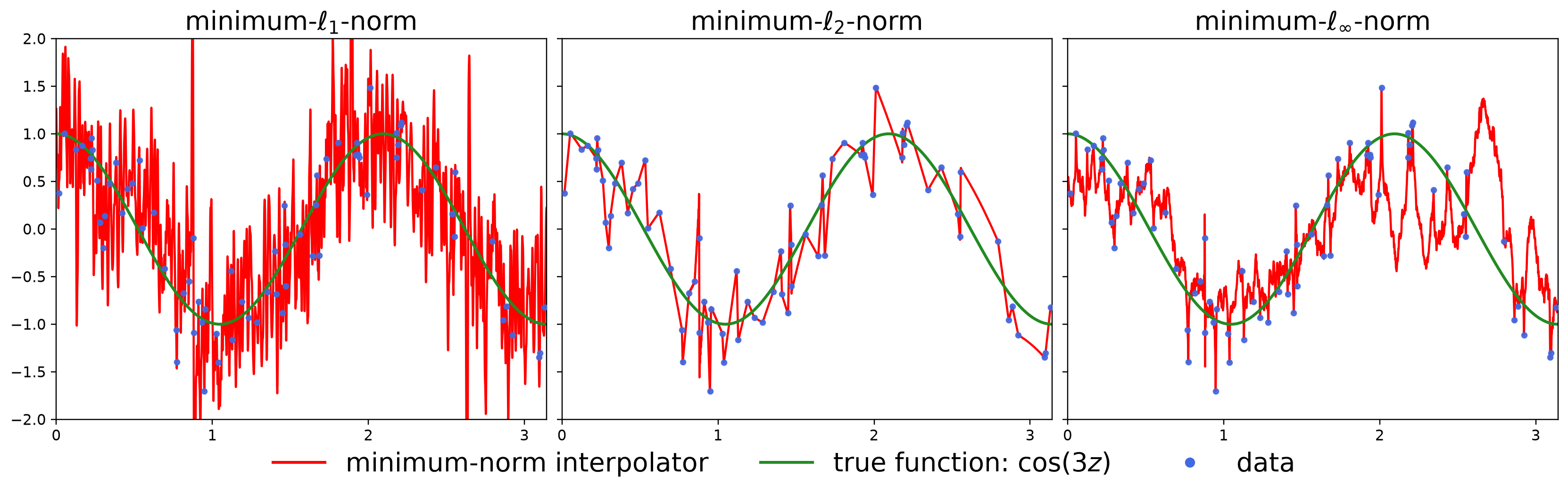}
    \vspace{-2em}
    \caption{Minimum-$\ell_p$ interpolation for $p\in\{1,2,\infty\}$, over i.i.d inputs $\bx \in \R^{2000}$ with coordinates $x_j = \cos(jz)/j$ where $z\sim \mathrm{Unif}([0, \pi])$.
    The number of inputs is $n=75$ and the labels are
    $y=\cos(3z)+\xi$, with $\xi\sim\mathcal N(0,0.1)$.
    The horizontal axis shows $z$.
    }
    \label{fig:cosine-norms}
    \vspace{-1em}
\end{figure*}

Unfortunately, many techniques developed for the minimum-\(\ell_2\)-norm interpolator crucially rely on it having a closed form, which is not typically true for other norms. For non-\(\ell_2\) norms, existing works focus on Gaussian inputs (and often also Gaussian noise), and use tools that crucially rely on this assumption, such as the Gaussian Minmax theorem
\citep{koehler2021uniform,kur2026gaussianity}. 
This is not just a technical limitation: As we will show later in this paper, the existence of benign overfitting can be very different for Gaussian and non-Gaussian distributions, even for well-studied settings such as minimum $\ell_1$-norm interpolation.
The previous discussions motivate the central question underpinning this work:
\begin{center}
        When and why do non-\(\ell_2\) minimum-norm interpolators exhibit benign overfitting in standard (possibly non-Gaussian) high-dimensional regression settings?
\end{center}

As in many previous works on benign overfitting (e.g. \cite{bartlett2020benign, koehler2021uniform, shamir2023implicit}), our analysis partitions the $d$-dimensional inputs into a small set $S$ of ``signal coordinates'' used to fit the targets, and the remaining set $N$ of ``noise-fitting'' coordinates used to fit the noise without incurring large excess risk. 
This split is used only in the analysis, and the algorithm does not rely on it. 
Prior work on the \(\ell_2\) interpolator
shows that benign overfitting can occur when $S$ is relatively low-dimensional and contains most of the target, while $N$ has a large effective dimension (meaning that there are many directions with similar variance).

Our contributions and paper structure can be summarized as follows:
\begin{itemize}[leftmargin=*]
    \item We begin in \secref{sec:main-results} by considering the minimum-\(\ell_p\)-norm interpolator for any fixed
    \(p\in(1,\infty]\). We prove a quantitative risk bound that implies benign overfitting, under suitable assumptions broadly similar to those in previous works (Theorem~\ref{thm:lp-clean} and Corollary~\ref{cor:lp-asymptotic}). Specifically, suppose that the input and label-noise are independent and sub-Gaussian, the noise-fitting
    coordinates have common variance \(s_N^2\), and the target vector $\bw^\star$
    is supported on \(S\). Writing
    \(q\) for the dual exponent, defined by \(1/p+1/q=1\), our results imply benign overfitting in the regime
    \begin{align*}
        s_N^2 |N|^{2/q}|S|^{\max\{2/p-1, 0\}}
        &\ll n \ll |N|^{\min\{2/q,1\}}.
    \end{align*}
    For $\ell_2$, this matches the necessary and sufficient criteria implied by \citet{bartlett2020benign}.

    \item In \secref{sec:mechanism}, we explain our method and main proof ideas for analyzing the minimum-norm interpolator. 
    Roughly speaking, our analysis begins by leveraging the dual objective of the min-norm interpolation problem. We then demonstrate that under suitable conditions on the norm (which we prove hold for $\ell_p$ norms when $p\in (1,\infty]$), and for high-dimensional inputs where the 
    coordinates are independent and sub-Gaussian, the dual problem can be well-approximated by a least-squares problem with $\ell_2$ geometry, leading to a behavior similar to the $\ell_2$ case. The techniques used in this approach, which include quantitative central limit  and uniform concentration arguments, may be of independent interest. 

    \item 
    In \secref{sec:l1-failure}, we show that benign overfitting \emph{does not} generally hold for minimum-\(\ell_1\)-norm interpolation. For scaled Rademacher inputs and suitable independent sub-Gaussian noise, the excess risk diverges as \(\Omega(\sigma^2\log n)\) with constant probability
    (Theorem~\ref{thm:l1-mean-zero-failure}).
    This is perhaps surprising in light of $\ell_1$ benign overfitting results for \emph{Gaussian} distributions \citep{koehler2021uniform,wang2022tight}. Nevertheless, we demonstrate that the geometry of the minimum-$\ell_1$-norm regression problem can be completely different between Gaussians and other well-behaved distributions. This is also illustrated in 
    \figref{fig:cosine-norms}, which shows a simple empirical example where the minimum-$\ell_1$ interpolator clearly behaves much worse than the minimum-$\ell_2$ and $\ell_\infty$-norm interpolators, in terms of generalization behavior.

    \item In \secref{sec:general-result}, we generalize our theory to provide risk bounds for generic norms (beyond $\ell_p$), and general sub-Gaussian distributions with independent coordinates (Theorem~\ref{thm:main}). 
    The theorem identifies two conditions on $N$ that suffice for benign overfitting, related to its expected dual norm on flat directions, and uniform concentration of the dual norm over the sphere. We prove these conditions indeed hold for \(\ell_p\) norms as a special case.
\end{itemize}
All proofs are provided in the appendices. We defer discussion of related work to \appref{sec:related-work}.

\section{Setting and Preliminaries}
\label{sec:preliminaries}

Consider an input distribution \(\Dcal_{\bx}\) on \(\R^d\). Let \(\bx_1,\ldots,\bx_n \sim \Dcal_{\bx}\) be i.i.d. input samples, forming a matrix \(X\in\R^{n\times d}\) with rows \(\bx_i^\top\). 
Let \(\bxi:=(\xi_1,\ldots,\xi_n)^\top\) be a noise vector, $\bw^\star \in \R^d$ be a target vector, and define the labels by \(\by:=X\bw^\star+\bxi\). 
Throughout the paper, unless stated otherwise, $\norm{\cdot}$ denotes an arbitrary norm on $\reals^d$.
For such a norm, the minimum-norm interpolator is any
\[
    \widehat{\bw}
    \in
    \argmin{\bw\in\R^d}
    \left\{
        \|\bw\|:X\bw=\by
    \right\}.
\]
For any $\bw \in \R^d$, we define its excess risk as
\[
    \Rcal(\bw)
    :=
    \E_{\bx\sim\Dcal_{\bx}}
    \left[
        \left(\bx^\top(\bw-\bw^\star)\right)^2
    \right].
\]
Our goal is to identify conditions under which $\Rcal(\widehat{\bw}) \to 0$ (where $\hat \bw$ is a minimum-norm interpolator), which is the standard notion of benign overfitting in the literature.

For our analysis, following previous works, we partition the input coordinates $\{1,\ldots,d\}$ into two nonempty sets \(S\sqcup N\), where \(S\) is the
\emph{signal block} (thought of as low dimensional, and ideally $\bw^\star$ is mostly supported on $S$) and \(N\) is the \emph{noise-fitting block} (thought of as high dimensional). 
After relabeling the coordinates, we may assume that $S=\{1,\ldots, |S|\}$ and $N=\{|S|+1,\ldots,d\}$.
This is purely an indexing convention and does not affect the algorithm in any way.

\paragraph{Assumptions.}
For a random variable \(Z\), write $\|Z\|_{\psi_2}$ for its sub-Gaussian norm\footnote{Defined by 
$ \|Z\|_{\psi_2}:=\inf\left\{t>0:
\E\exp\left(Z^2 / t^2 \right)\le 2\right\}$
}, 
and call $Z$ sub-Gaussian if $\|Z\|_{\psi_2}<\infty$
(see \citet{Vershynin_2026} for examples, e.g. bounded or Gaussian variables).
We assume that the inputs and noise have independent, zero-mean sub-Gaussian coordinates:

\begin{assumption}[Sub-Gaussian input distribution]
\label{assump:subgaussian-input-distribution}
The inputs \(\bx_i=(x_{i,1},\ldots,x_{i,d})\sim\Dcal_{\bx}\) have independent
coordinates, and there are \(b_1,\ldots,b_d>0\) and
\(R_X > 0\) such that for every \(j\in[d]\),
\[
    \E[x_j]=0,
    \qquad
    \E[x_j^2]=b_j^2,
    \qquad
    \|x_j\|_{\psi_2}\le R_Xb_j.
\]
\end{assumption}
\begin{assumption}[Sub-Gaussian noise]
\label{assump:noise-variables}
    The noise vector \(\bxi=(\xi_1,\ldots,\xi_n)^\top\) has independent
    coordinates that are also independent of \(X\). There are
    \(\sigma>0\) and \(R_\xi > 0\) such that for every \(i\in[n]\),
    \[
        \E[\xi_i]=0,
        \qquad
        \E[\xi_i^2]=\sigma^2,
        \qquad
        \|\xi_i\|_{\psi_2}\le R_\xi\sigma.
    \]
\end{assumption}

Assumptions \ref{assump:subgaussian-input-distribution} and
\ref{assump:noise-variables} are rather standard in the literature.
While we adopt them for simplicity, as \secref{sec:mechanism} clarifies, our analysis can extend to milder assumptions, but this may impact the exact bounds.

\paragraph{Notation and additional definitions.} For $m\in \N$ let $[m]:=\{1,\ldots, m\}$. 
For \(T\subseteq[d]\), we write \(\bv_T\) for the restriction of a
vector \(\bv\) to the coordinates in \(T\), and \(X_T\) for the submatrix of
\(X\) with columns indexed by \(T\). For a partition $d=S\sqcup N$ of the input coordinates  write
\[
    \bw=(\bw_S,\bw_N),
    \qquad
    \bw^\star=(\bw_S^\star,\bw_N^\star),
    \qquad
    X=(X_S,X_N).
\]
Moreover, for any nonempty $T\subseteq [d]$, write:
\[
    b_T^-:=\min_{j\in T}b_j,
    \qquad
    b_T^+:=\max_{j\in T}b_j,
    \qquad
    \Sigma_T
    :=
    \E[\bx_T\bx_T^\top]
    =
    \operatorname{diag}\bigl((b_j^2)_{j\in T}\bigr),
\]
and define the effective rank
\[
    r_T
    :=
    \frac{\operatorname{tr}(\Sigma_T)}{\|\Sigma_T\|_{\op}}
    =
    \frac{\|\bb_T\|_2^2}{(b_T^+)^2}.
\]

\section{Benign Overfitting for \texorpdfstring{$\ell_p$}{lp} Norms (\texorpdfstring{$p \in (1, \infty]$}{p > 1})}
\label{sec:main-results}
For any \(p\in(1,\infty]\), let \(q\in[1,\infty)\) be its dual exponent, given by \(1/p+1/q=1\). To reduce the amount of notation in this section, we will make two simplifying assumptions, which will both be relaxed in Section~\ref{sec:general-result}. First, we assume that the coordinates in $N$ share the same variance. While this is not strictly necessary, it simplifies the final bound while capturing the important underlying phenomenon. 
Moreover, it is well-known that for benign overfitting to occur, a necessary condition (in the $\ell_2$ analyses and for general norms under Gaussian inputs) is essentially that there are many coordinates with similar variance \citep{bartlett2020benign, koehler2021uniform}.
Second, we assume that the target vector $\bw^\star$ is supported on the $S$ coordinates, meaning $\bw^\star = (\bw_S^\star, 0)$. 
For a more general target vector, the risk bounds would incur additional factors related to $\|\Sigma_N^{1/2}\bw_N^\star\|_2$. Once again, such terms must be small anyway for benign overfitting to generally occur \citep{bartlett2020benign, koehler2021uniform}. 
We are now ready to state our theorem:
    
\begin{theorem}[Risk bound for \(\ell_p\) norms]
\label{thm:lp-clean}
Let $[d] = S \sqcup N$ be any partition of coordinates, and suppose Assumptions~\ref{assump:subgaussian-input-distribution}
and~\ref{assump:noise-variables} hold, that
\(\bw^\star=(\bw_S^\star,\zero)\) and that
\((x_j)_{j\in N}\) have common variance \(s_N^2\). For any $p\in (1,\infty]$, there is a positive
\(C=\poly\!\left(R_X,R_\xi,\sigma,\sigma^{-1},(b_S^-)^{-1},b_S^+,\|\bw_S^\star\|_p \right)\)
such that, for every
\(\delta\in(0,1)\), with probability at least \(1-\delta\) every minimum-\(\ell_p\)-norm
interpolator satisfies
\begin{align*}
    \Rcal(\widehat\bw)
    \le
    C\left[
        \frac{s_N|N|^{\frac{1}{q}}|S|^{\max\left\{\frac{1}{p}-\frac{1}{2},0\right\}}}{\sqrt n}
        +
        \left(\frac{r_S+\log(\frac{n}{\delta})}{n}\right)^{\frac{1}{4}}
        +
        \sqrt{\frac{n\log\left(\frac{n|N|}{\delta}\right)}{|N|}} + 
        \frac{n\log\left(\frac{n|N|}{\delta}\right)}{|N|^{\min\left\{\frac{2}{q},1\right\}}}
    \right],
\end{align*}
provided that the quantity in square brackets is at most \(C^{-1}\). Moreover, $C$ is non-increasing in $p$. 
\end{theorem}

We now provide an asymptotic corollary that clarifies the conditions for benign overfitting. 
To formally avoid issues related to the order of limits, take $n\to\infty$ and view all other parameters (e.g. $N$, $s_N$, $\delta$, $S$) as functions of $n$, yielding the following corollary:
\begin{corollary}[Benign overfitting for \(\ell_p\) norms]
\label{cor:lp-asymptotic}
Consider a sequence of problem instances indexed by the sample size \(n\), each satisfying the assumptions of
Theorem~\ref{thm:lp-clean}. Suppose that
\(R_X,R_\xi,\sigma,\sigma^{-1},(b_S^-)^{-1},b_S^+\), and \(\|\bw_S^\star\|_p\) are uniformly bounded.
If there is a sequence \(\delta_n\to0\) such that
\(\frac{r_S+\log(n/\delta_n)}{n}\to0\) and if
\[
\begin{aligned}
    \frac{\bigl(n\log(n|N|/\delta_n)\bigr)^{\max\{1,q/2\}}}{|N|}
        &\longrightarrow 0,
    \qquad
    \frac{s_N|N|^{1/q}|S|^{\max\{1/p-1/2,0\}}}{\sqrt n}
        &\longrightarrow0,
\end{aligned}
\]
then, with probability tending to one, every minimum-\(\ell_p\)-norm
interpolator satisfies
\[
    \Rcal(\widehat\bw)\longrightarrow0.
\]
\end{corollary}

Analogously to related settings studied in prior work, the regime in which the results are meaningful is when relatively few coordinates can fit the targets, and many low-variance coordinates are used to interpolate the noise. Accordingly, the asymptotic conditions proved sufficient by Theorem~\ref{thm:lp-clean} and Corollary~\ref{cor:lp-asymptotic} for benign overfitting are
\begin{align}\label{eq:criteria}
    r_S\ll n~, 
    \qquad \text{and} \qquad 
    s_N^2 |N|^{2/q}|S|^{\max\{2/p-1, 0\}} &\ll n \ll |N|^{\min\{2/q,1\}}.
\end{align}
Since $r_S \leq |S|$, the first condition is always satisfied when $|S| \ll n$. The \(n \ll |N|^{\min\{2/q, 1\}}\) becomes increasingly demanding as $p\to 1$, but as we will later show in \secref{sec:l1-failure}, this is not an artifact of our analysis: benign overfitting can fail for \(p=1\) even under the assumptions of Theorem \ref{thm:lp-clean}.

In the $\ell_2$ case and under the setting of Theorem~\ref{thm:lp-clean} and Corollary~\ref{cor:lp-asymptotic}, the asymptotic criteria for benign overfitting in our results match those of \cite{bartlett2020benign}, which were proved to be both necessary and sufficient. We remark that their results are stated in terms of various notions of effective rank, which can be computed explicitly in this setting to match \eqref{eq:criteria}. Likewise, the criteria for benign overfitting given in \eqref{eq:criteria} match those of \cite{koehler2021uniform} for any $p\in [2,\infty]$ (although the exact convergence rates vary). However, unlike our results, that paper crucially requires the inputs to be Gaussian, e.g. by relying on the Convex Gaussian Minmax Theorem \citep{thrampoulidis2015regularized}.


The conditions of Theorem~\ref{thm:lp-clean} and Corollary~\ref{cor:lp-asymptotic} are sufficient but not necessary. We will later prove in \secref{sec:general-result} a more general version (Theorem~\ref{thm:main}) that holds for general norms and drops both the common-variance and $\bw^\star$-support assumptions.

\section{Proof Ideas}
\label{sec:mechanism}
We now outline the techniques underlying the results. We focus here on the $\ell_p$ norm case, but as we will further discuss in \secref{sec:general-result}, the proof can readily be adapted to more general norms. 

\subsection{Warm-Up: Gaussian Inputs}
For simplicity, we first consider the case where the inputs $\bx$ are Gaussian, as it provides some useful intuition, and then detail how the framework extends to other distributions. 

Fix \(p\in(1,\infty)\), and let \(q\) be its dual exponent, defined by \(1/p+1/q=1\) (our proof also handles $p=\infty$, but we omit it here for simplicity).
Write $\bw = (\bw_S, \bw_N)$. We decompose the minimum-norm problem by first minimizing over $\bw_S$, and then minimizing $\bw_N$ as a function of $\bw_S$:
\begin{align}\label{eq:mechanism-min-problem}
    \min_{\bw : X\bw=\by}\|\bw\|_p
    &= 
    \min_{\bw : X\bw=\by}
    \left[\|\bw_S\|_p^p+\|\bw_N\|_p^p\right]^{1/p}\notag\\
    &=\min_{\bw_S \in \R^S} 
    \left[\|\bw_S\|_p^p + \min_{\bw_N:X_N\bw_N=\by - X_S\bw_S} \|\bw_N\|_p^p \right]^{\frac{1}{p}}.
\end{align}
We first focus on the inner minimization over $\bw_N$, and later show that under the conditions of the theorem, this is the dominating term. 
We define the cost of fitting $\bw_N$ and the residual as:
\[
    g(\bw_S):=
    \min_{\bw_N:X_N\bw_N=\bz(\bw_S)}
    \|\bw_N\|_p,
    \qquad
    \bz(\bw_S):=\by - X_S\bw_S = X_S(\bw_S^\star-\bw_S)+\bxi.
\]
A central part of our approach is to analyze $g(\bw_S)$ through convex duality. Specifically, the minimization problem defining $g(\bw_S)$ can equivalently be written in its dual form as:
\begin{equation}
\label{eq:mechanism-duality}
    g(\bw_S)
    =
    \max_{\blambda \in \R^n : \|X_N^\top\blambda\|_q\le1}
    \bz(\bw_S)^\top\blambda.
\end{equation}
Let $B_n(r):=\{\blambda \in \R^n: \norm{\blambda}_2\leq r\}$ be the $\ell_2$ ball of radius $r$. We will now show that the feasible set in \eqref{eq:mechanism-duality} is well-approximated by $B_n(r)$ for a suitably chosen $r$. 
Specifically, when the rows of \(X_N\) are i.i.d. copies of a Gaussian vector \(\bx_N\), then, writing \(\kappa_N:=\E\|\bx_N\|_q\),
for every \(\blambda\in\R^n\),
\begin{equation}\label{eq:mechanism-isotropy2}
    X_N^\top\blambda
    \stackrel{d}{=}
    \|\blambda\|_2\bx_N
    \qquad\Longrightarrow\qquad
    \E\|X_N^\top\blambda\|_q
    =
    \kappa_N\|\blambda\|_2.
\end{equation}
When $\bx_N$ is high-dimensional, one can
further show that, with high probability,
\(\|X_N^\top\blambda\|_q\) is close to its expectation uniformly over
\(\blambda\), meaning that
\begin{align}\label{eq:mechanism-concentration}
    (1-\varepsilon)\kappa_N\|\blambda\|_2
    \le
    \|X_N^\top\blambda\|_q
    \le
    (1+\varepsilon)\kappa_N\|\blambda\|_2
    \qquad
    \forall\,\blambda\in\R^n,
\end{align}
for some small \(\varepsilon>0\). On this event,
\[
B_n\left(\frac{1}{(1+\varepsilon)\kappa_N}\right) \subseteq \{\blambda \in \R^n : \|X_N^\top\blambda\|_q\le1\} \subseteq B_n\left(\frac{1}{(1-\varepsilon)\kappa_N}\right).
\]
Thus, the dual problem in
\eqref{eq:mechanism-duality} is maximizing a linear function over an approximate Euclidean ball, from which it readily follows that
\begin{equation}
\label{eq:mechanism-gaussian-g-bound}
    \frac{\|\bz(\bw_S)\|_2}{(1+\varepsilon)\kappa_N}
    \le
    g(\bw_S)
    \le
    \frac{\|\bz(\bw_S)\|_2}{(1-\varepsilon)\kappa_N}~.
\end{equation}
At the same time, recalling that $\bz(\bw_S):=X_S(\bw_S^\star-\bw_S)+\bxi$ and that the noise $\xi_i$ has mean-zero and variance $\sigma^2$, it holds for sufficiently large $n$ that with high probability, the residuals $\bz(\bw_S)$ can be well approximated as
\begin{align}\label{eq:mechanism-res-l2}
    \frac1n\|\bz(\bw_S)\|_2^2
    =& ~
    (\bw_S - \bw_S^\star)^\top\left(\frac1nX_S^\top X_S\right)(\bw_S - \bw_S^\star)
    -\frac2n(\bw_S - \bw_S^\star)^\top X_S^\top\bxi
    +\frac1n\|\bxi\|_2^2 \nonumber\\
    \approx& ~
    \norm{\Sigma_S^{1/2}(\bw_S - \bw_S^\star)}_2^2 + \sigma^2.
\end{align}
Overall, we have shown that for any fixed $\bw_S$, 
\begin{align}\label{eq:mechanism-N-bound}
    \min_{\bw_N:X_N\bw_N=\by - X_S\bw_S} \|\bw_N\|_p
    \approx \frac{\|\bz(\bw_S)\|_2}{\kappa_N}
    \approx 
    \frac{\sqrt{n}}{\kappa_N}\cdot\sqrt{\sigma^2 + \norm{\Sigma_S^{1/2}(\bw_S - \bw_S^\star)}_2^2}.
\end{align}
Substituting this into \eqref{eq:mechanism-min-problem}, the original min-norm problem is roughly equivalent to:
\[
    \min_{\bw : X\bw = \by} \norm{\bw}_p
    \approx
    \min_{\bw_S \in \R^S}\left[\|\bw_S\|_p^p
    +
    \left(\frac{n}{\kappa_N^2}\right)^{p/2}\left(\sigma^2 + \norm{\Sigma_S^{1/2}(\bw_S - \bw_S^\star)}_2^2\right)^{p/2}\right]^{\frac{1}{p}}.
\]
When $n\gg \kappa_N^2$, the dominating term that we can minimize is $\|\Sigma_S^{1/2}(\bw_S - \bw_S^\star)\|_2^2$, so the min-norm problem is roughly equivalent to a least-squares problem on the signal block,
\[
\min_{\bw_S} \norm{\Sigma_S^{1/2}(\bw_S - \bw_S^\star)}_2^2~.
\]
Therefore, recalling that $\kappa_N := \E[\norm{\bx_N}_q] \approx s_N|N|^{1/q}$, whenever $n\gg \kappa_N^2\approx s_N^2|N|^{2/q}$ and $\Sigma_S$ is well-conditioned, the minimum-$\ell_p$-norm solution $\hat \bw$ must approximate $\bw^\star$ on the signal block: 
\[
\hat \bw_S \approx \bw_S^\star.
\]
It remains to show that the effect of the noise-fitting block on the population error is negligible. To see this, when $\hat \bw_S \approx \bw_S^\star$, \eqref{eq:mechanism-N-bound} implies that $\norm{\hat \bw_N}_p \approx \frac{\sqrt{n}}{\kappa_N}\sigma$. Using the standard inequality between the $\ell_2$ norm and the $\ell_p$ norm, and assuming that $\bw_N^\star = 0$,
\begin{align*}
\|\Sigma_N^{1/2}(\widehat\bw_N-\bw_N^\star)\|_2^2 
\leq & \norm{\Sigma_N}_\op \norm{\widehat\bw_N}_2^2 
\leq |N|^{\max\{1 - \frac{2}{p},0\}}\norm{\Sigma_N}_\op \norm{\widehat\bw_N}_p^2 \\ 
\approx& |N|^{\max\{1 - \frac{2}{p},0\}}\norm{\Sigma_N}_\op \frac{n}{\kappa_N^2}\sigma^{2}.
\end{align*}
Finally, since $\kappa_N\asymp s_N|N|^{1/q}$ and $\norm{\Sigma_N}_\op=s_N^2$, the last display is of order $\sigma^2n/|N|^{\min\{2/q,1\}}$. Thus, ignoring logarithmic factors, it converges to $0$ when $n\ll |N|^{\min\{2/q,1\}}$.
In such a case, using $\hat \bw_S\approx \bw_S^\star$, the excess risk is small:
\[
    \Rcal(\widehat\bw) 
    =
    \|\Sigma_S^{1/2}(\widehat\bw_S-\bw_S^\star)\|_2^2 
    +
    \|\Sigma_N^{1/2}(\widehat\bw_N-\bw_N^\star)\|_2^2
    \approx 0.
\]

\subsection{Beyond Gaussians}
For non-Gaussian inputs, the high-level proof approach is similar, but several arguments from the Gaussian proof require adaptation and additional ideas. 
The first is that for general input distributions, the rotational invariance used in \eqref{eq:mechanism-isotropy2} no longer holds. 
Nevertheless, we can get an approximate version using a central limit theorem (CLT) argument. 
To see this, first note that the vectors $X_N^\top\blambda$ which determine the dual constraints in \eqref{eq:mechanism-duality} have coordinates $(X_N^\top\blambda)_j=\sum_{i=1}^n\lambda_i x_{i,j}.$
If one had $\lambda_i = 1/\sqrt{n}$ for every $i$, the CLT would let us approximate $(X_N^\top \lambda)_j$ with a suitable Gaussian. 
It is well known that in some cases, similar results hold when \(\blambda\) is ``flat'', 
in the sense that \(\|\blambda\|_\infty \ll \|\blambda\|_2\). 
Lemma~\ref{lem:K_subgaussian} in the appendix makes this idea quantitative, and leads to a statement analogous to \eqref{eq:mechanism-isotropy2} for flat $\blambda$.  
Specifically, letting $\kappa_N := s_N\left(\abs{N}\E_{G \sim \Ncal(0, 1)}[\abs{G}^q]\right)^{1/q}$, the lemma states that, for every \(\blambda\in\R^n\),
\[
    \left|
        \frac{\E\|X_N^\top\blambda\|_q}{\kappa_N\|\blambda\|_2}-1
    \right|
    \lesssim
    |N|^{-1/(2q)}+
    \frac{\|\blambda\|_\infty}{\|\blambda\|_2}.
\]
Additionally, Proposition~\ref{prop:xnl_conc_subgaussian} in the appendix shows that $\|X_N^\top\blambda\|_q$ is close to its expectation via uniform concentration, allowing us to obtain a result analogous to \eqref{eq:mechanism-concentration}, namely that with high probability
\[
    \sup_{\blambda\in\Sphere^{n-1}}
    \left|
        \frac{\|X_N^\top\blambda\|_q}
        {\E\|X_N^\top\blambda\|_q}
        -1
    \right|
    \lesssim
    \sqrt{\frac{n}{|N|}}
    +
    \frac{n^{q/2}}{|N|}
    +
    |N|^{-1/(2q)}
\]
(ignoring log factors). 
In principle, combining the last two displayed equations implies that $\|X_N^\top\blambda\|_q$ tends to be close to $\kappa_N\|\blambda\|_2$, leading to the same Euclidean geometry as before and allowing our argument to go through. The remaining challenge is that the $\blambda$ maximizing the dual problem in \eqref{eq:mechanism-duality} does not necessarily need to be flat (in which case  $\frac{\|\blambda\|_\infty}{\|\blambda\|_2}$ is not small, as required for the CLT argument). 
Nevertheless, as we will soon explain in more detail, Lemma~\ref{lem:truncate} in the appendix shows that we can find a \emph{close to optimal} $\blambda$ that is flat (and thus the previous CLT argument applies) whenever $\bz(\bw_S)$ is flat. Moreover, we show that the existence of a near-optimal flat $\blambda$ is enough for the argument to go through. Specifically, we first show that $\bz(\bw_S)$ is flat:  Due to the independent sub-Gaussian noise variables, we can show that up to some error terms,
\[
    \frac{\|\bz(\bw_S)\|_\infty}{\|\bz(\bw_S)\|_2}\lesssim \sqrt{\frac{\log(n)}{n}}.
\]
Indeed, proving that the denominator scales as \(\sqrt n\) is done as in \eqref{eq:mechanism-res-l2}, and the numerator scale of \(\sqrt{\log n}\) arises from the maximum of $n$ independent noise variables.
Next, for every feasible $\blambda$, one can zero out all coordinates bigger than some threshold $\tau>0$, giving $\blambda_\tau := \sum_{j=1}^n \one_{\abs{\lambda_j}\leq \tau} \lambda_j \be_j$. Then we prove that for the dual objective given in \eqref{eq:mechanism-duality}, after slight rescaling, $\blambda_\tau$ is feasible and close to optimal when $\|\bz(\bw_S)\|_\infty \ll \|\bz(\bw_S)\|_2$:
\begin{align}\label{eq:trunc-lambda}
    \bz(\bw_S)^\top \blambda_\tau \gtrsim \bz(\bw_S)^\top \blambda - \frac{\norm{\bz(\bw_S)}_\infty}{\tau} \gtrsim \bz(\bw_S)^\top \blambda - \frac{\norm{\bz(\bw_S)}_2}{\tau}\sqrt{\frac{\log(n)}{n}}.
\end{align}
There is now a tradeoff: a smaller choice of $\tau$ is better for the quantitative CLT argument, but causes the gap between $\bz(\bw_S)^\top \blambda_\tau$ and $\bz(\bw_S)^\top \blambda$ to increase. Taking $\tau \approx n^{-1/4}$ balances both, meaning that $\blambda_{\tau}$ both satisfies the assumptions needed for the CLT argument, and is an approximate maximizer of the dual objective given in \eqref{eq:mechanism-duality}.

Combining the ingredients discussed in this subsection allows us to obtain a result analogous to
\eqref{eq:mechanism-gaussian-g-bound}, up to the error terms appearing in the
theorem. The remaining details of the proof transfer from the Gaussian case more straightforwardly.

\section{A Failure Example for Minimum-\texorpdfstring{$\ell_1$}{l1}-Norm Interpolation}
\label{sec:l1-failure}
Our risk bounds in \thmref{thm:lp-clean} give sufficient conditions for benign overfitting for any $\ell_p$ norm with $p\in(1,\infty]$, but not for $\ell_1$. Indeed, the concentration and CLT arguments (described in \secref{sec:mechanism}), which showed that the dual optimization problem is approximately Euclidean, do not hold when $p=1$. One may wonder if this gap is a proof artifact or reflects genuinely different behavior of the minimum-$\ell_1$-norm predictor.
As noted in the introduction, it is known that over \emph{Gaussian} inputs (with suitable covariance), the minimum-$\ell_1$ norm does exhibit benign overfitting  \citep{ju2020overfitting,koehler2021uniform, wang2022tight}, with the excess risk scaling as $\sigma^2/\log(d/n)\rightarrow 0$ when $d/n \to \infty$. 
Perhaps surprisingly, we show that despite such results for Gaussians, minimum $\ell_1$ interpolation does \emph{not} generally enjoy benign overfitting. In particular, even for well-behaved mean-zero sub-Gaussian inputs and noise (satisfying Assumptions~\ref{assump:subgaussian-input-distribution} and~\ref{assump:noise-variables}), the risk of the minimum-$\ell_1$-norm predictor can diverge to infinity:

\begin{theorem}[Failure of minimum-\(\ell_1\)-norm interpolation]
\label{thm:l1-mean-zero-failure}
Fix \(n\ge 2\), $d > 2n2^n$, $\sigma^2 > 0$, and inputs 
\[
    x_1\sim\operatorname{Unif}\{\pm1\},
    \qquad
    \text{and}
    \qquad
    x_j\sim\operatorname{Unif}\{\pm s_N\}
    \quad \forall ~ j\in \{2,\ldots, d\} \quad \text{where} ~~ s_N\in(0,1),
\]
where coordinates are independent (and sub-Gaussian). There exists independent sub-Gaussian noise $\xi$ with $\Var(\xi) = \sigma^2$ such that for $\bw^\star=\zero$, with probability at least \(1/5\),
every minimum-\(\ell_1\)-norm interpolator satisfies
\[
    \Rcal(\widehat\bw)\ge \sigma^2 \frac{\log n}{9}.
\]
Moreover, all distributional and structural assumptions in \thmref{thm:lp-clean} are satisfied: 
for $S=\{1\}, N=\{2,\ldots, d\}$, Assumptions~\ref{assump:subgaussian-input-distribution} and~\ref{assump:noise-variables} hold with $b_j=s_N$ for all $j\in N$, and with
\(R_X,R_\xi\) bounded by universal constants, uniformly in
\(n\), \(|N|\), and \(s_N\).
\end{theorem}
The construction in \thmref{thm:l1-mean-zero-failure} uses dimension exponential in the number of samples. Over Gaussian inputs, the results of \citet{koehler2021uniform} yield benign overfitting for this same covariance and exponentially large dimension. 

We now give intuition as to how the change from Gaussians to (scaled) Rademacher inputs completely changes the nature of the minimum-$\ell_1$ interpolation problem. As in the proof of our positive results (cf. \secref{sec:mechanism}) we begin by decomposing the minimum-\(\ell_1\)-norm interpolation problem as
\begin{align}\label{eq:sketch_l1_problem}
    \min_{\bw:\,X\bw=\by}\|\bw\|_1
    =
    \min_{w_S \in \R} 
        \left(
            |w_S|+ \min_{\bw_N: X_N\bw_N=\by - X_Sw_S}\|\bw_N\|_1
        \right).
\end{align}
Now suppose that the inputs have $N$ coordinates in $\{\pm s_N\}$ and with high probability every vector in \(\{\pm s_N\}^n\) appears as a column of $X_N$. 
On this event, for any $\blambda \in \R^n$,
\begin{align}\label{eq:no-istotropy}
    \|X_N^\top\blambda\|_\infty
    =s_N\max_{\bu\in\{\pm1\}^n}|\bu^\top\blambda|
    =s_N\|\blambda\|_1.
\end{align}
Using \(\ell_1\) duality and letting $\bz(w_S):=\by - X_Sw_S$, for any $w_S$ the inner minimization problem in \eqref{eq:sketch_l1_problem} becomes
\[
\min_{\bw_N: X_N\bw_N=\bz(w_S)}\|\bw_N\|_1 = \max_{\|X_N^\top\blambda\|_\infty\le1} \bz(w_S)^\top\blambda
=
\max_{s_N\|\blambda\|_1\le1}\bz(w_S)^\top\blambda     
= \frac{1}{s_N}\|\bz(w_S)\|_\infty.
\]
Thus, any min-norm interpolator satisfies
\begin{align}\label{eq:sketch_massaged}
    \hat w_S
    &\in
    \argmin{w_S \in \R}
        \left(
            |w_S|+\frac{1}{s_N}\|\by-X_S w_S\|_\infty
        \right).
\end{align}
By contrast, if the rows of \(X_N\) were i.i.d. copies of a Gaussian vector \(\bx_N\), 
then since $X_N^\top\blambda\stackrel{d}{=}\|\blambda\|_2\bx_N$, \eqref{eq:no-istotropy} would be replaced with $\|X_N^\top\blambda\|_\infty \stackrel{d}{=} \norm{\blambda}_2\norm{\bx_N}_\infty$.
When $|N|$ is large, there is a deterministic value $\kappa_N > 0$ (which may depend on $N$) such that with high probability, $\norm{\blambda}_2\norm{\bx_N}_\infty \approx \norm{\blambda}_2\kappa_N$ for all $\blambda$,
and consequently one can prove that:
\begin{align}\label{eq:sketch_massaged_v2}
    \hat w_S
    \approx 
    \argmin{w_S \in \R}
        \left(
            |w_S|+\frac{1}{\kappa_N}\|\by-X_S w_S\|_2
        \right).
\end{align}
The two problems in \eqref{eq:sketch_massaged} and \eqref{eq:sketch_massaged_v2} have fundamentally different geometries: 
In the Rademacher case, the solution minimizes the $\ell_\infty$ (maximal) residual error, whereas in the Gaussian case, the solution \eqref{eq:sketch_massaged_v2} minimizes the $\ell_2$ (average) error, analogously to our positive results for $\ell_p$ with $p\in(1,\infty]$ (\secref{sec:mechanism}).
Our negative result reflects the intuition that minimizing the $\ell_\infty$ error is inherently brittle and generally inconsistent. \figref{fig:dual-geometry} in the appendix visualizes this geometric contrast directly. 

\section{General Benign Overfitting Result}
\label{sec:general-result}
In \secref{sec:mechanism} we presented our framework behind the benign overfitting results for $\ell_p$ norms. However, our proof technique does not crucially rely on the specific structure of these norms, and in this section,
we generalize the results of \secref{sec:mechanism} to allow for more general norms, input distributions, and target vectors.
First, for any \(\bv\in\R^{N}\), define the dual of the induced \(N\)-block norm\footnote{While this notation is natural as $\norm{\bv}_*$ is the dual norm of $\bv \mapsto \norm{(\zero, \bv)}$, we note that it may in general not be the same as \emph{first} taking the dual norm on $\R^d$ and \emph{then} restricting to $\R^N$.} by
\[
    \|\bv\|_{*}
    :=
    \sup_{\|(\zero,\bu)\|\le1}\bv^\top\bu.
\]
In the case of $\ell_p$ norms, $\|\bv\|_{*} = \|\bv\|_q$ (where $1/p + 1/q=1$).
As in \secref{sec:mechanism}, we will split the minimum-norm interpolation problem into two: first, consider an arbitrary $\bw_S$ and then minimize the norm of $\bw_N$ under the constraint that $(\bw_S, \bw_N)$ interpolates the training data. The majority of our proof concerns the second minimization problem, whose dual form is given by
\[
    \min_{\bw_N:X_N\bw_N=\by - X_S\bw_S} \|(\zero, \bw_N)\|
    =
    \max_{\blambda \in \R^n : \|X_N^\top\blambda\|_*\le1}
    (\by - X_S\bw_S)^\top\blambda.
\]
We now state two assumptions that capture the properties of $\|X_N^\top\blambda\|_*$ that were needed for our proofs in \secref{sec:main-results} (and which are explicitly proven to hold under the conditions of \thmref{thm:lp-clean}). The first assumption is basically that an approximate CLT holds. 
Concretely, note that for \(X_N\) whose rows are i.i.d. copies of a Gaussian vector \(\bx_N\), then as explained in \eqref{eq:mechanism-isotropy2}, writing \(\kappa_N:=\E\|\bx_N\|_*\) gives $\E[\|X_N^\top\blambda\|_*] = \kappa_N \norm{\blambda}_2$. For more general distributions, by CLT considerations, it is thus reasonable to assume that $\E[\|X_N^\top\blambda\|_*] \approx \kappa_N \norm{\blambda}_2$ for a suitable $\kappa_N$ and for all $\blambda$ that are flat (in the sense that $\norm{\blambda}_\infty \ll \norm{\blambda}_2$). 

\begin{assumption}
\label{assump:K}
There exist a scaling factor \(\kappa_N>0\), a constant \(A>0\), and a
nondecreasing function \(\varphi:[0,\infty)\to[0,\infty)\) such that for every $\blambda\in\R^n$,
\[
    K(\blambda)
    :=
    \frac{\E\bigl[\|X_N^\top\blambda\|_{*}\bigr]}{\kappa_N} < \infty
\]
and for every \(\blambda\in\Sphere^{n-1}\):
    \[
        K(\blambda)\ge A 
        \qquad \text{and} \qquad 
        |K(\blambda)-1|
        \le
        \varphi(\|\blambda\|_\infty).
    \]
\end{assumption}
The lower bound with $A$ is a regularity assumption, and the exact value of $A$ is not crucial for the results. 
The next assumption we need is a uniform concentration assumption. 

\begin{assumption}
\label{assump:xnl_concentrate}
For every \(\delta\in(0,1)\), there is a deterministic
\(\varepsilon_{\mathrm{unif}}(\delta)\ge0\) such that, with probability at least
\(1-\delta\),
\[
    \sup_{\blambda\in\Sphere^{n-1}}
    \left|
        \frac{\|X_N^\top\blambda\|_{*}}
        {\E\|X_N^\top\blambda\|_{*}}
        -1
    \right|
    \le
    \varepsilon_{\mathrm{unif}}
    (\delta).
\]
\end{assumption}

We note that even for Gaussian inputs, uniform concentration over \emph{arbitrary} norms is a nontrivial question which has been extensively studied, e.g. via Dvoretzky's theorem and almost-Euclidean sections of convex bodies \citep{artstein2015asymptotic}. For many ``common'' norms, however, including $\ell_p$ ($p\in(1,\infty]$), uniform convergence does hold, as we show in Proposition~\ref{prop:xnl_conc_subgaussian}.

To relate the squared error to the norm $\norm{\cdot}$, we define
\[
    \beta_S
    :=
    \sup_{\bv\in\R^{S}\setminus\{0\}}
    \frac{\|(\bv,\zero)\|}{\|\bv\|_2},
    \quad
    \beta_{N,\Sigma}
    :=
    \sup_{\bv\in\R^{N}\setminus\{0\}}
    \frac{(\bv^\top\Sigma_N\bv)^{1/2}}{\|(\zero,\bv)\|},
    \quad
    \nu_N
    :=
    \bigl((\bw_N^\star)^\top\Sigma_N\bw_N^\star\bigr)^{1/2}.
\]
We are now ready to state our general results:

\begin{theorem}[Risk bound for general norms]
\label{thm:main}
Let $[d] = S \sqcup N$ be any partition of coordinates, and suppose Assumptions~\ref{assump:subgaussian-input-distribution},
\ref{assump:noise-variables}, \ref{assump:K}, and
\ref{assump:xnl_concentrate} hold. There exist \(C,C'\ge1\) and \(c>0\), depending (polynomially) only on \(A^{-1}\), such that the following holds for every
\(\delta\in(0,1)\). Suppose that $n\ge C\left(r_S+\log(n/\delta)\right)$ and
define \(\gamma_N:=\kappa_N/\sqrt n\) and
\begin{equation}
\label{eq:main-scales}
\begin{aligned}
    \rho
    &:=\frac{\nu_N}{\sigma}
    +\frac{R_X\|\bb_S\|_2+(R_Xb_S^++R_\xi\sigma)\sqrt{\log(n/\delta)}}{\sigma\sqrt n},
    \\
    \varepsilon_{\mathrm{flat}}
    &:=\sqrt\rho+\varepsilon_{\mathrm{unif}}(\delta/6)
    +\varphi(C'\sqrt\rho),
\end{aligned}
\end{equation}
and
\begin{equation}
\label{eq:main-signal-radius}
\begin{aligned}
    r_\delta^2
    :=\frac{C}{(b_S^-)^2}\Bigg[&
        \sigma^2\varepsilon_{\mathrm{flat}}
        +\gamma_N\|(\bw_S^\star,\zero)\|
        \left(\sigma+\gamma_N\|(\bw_S^\star,\zero)\|\right)
        +\sigma^2\left(\frac{\gamma_N\beta_S}{b_S^-}\right)^2\\
        &+
        \left(\frac{R_XR_\xi\sigma b_S^+}{b_S^-}\right)^2
        \sqrt{\frac{r_S+\log(n/\delta)}{n}}
    \Bigg].
\end{aligned}
\end{equation}
If \(r_\delta\le1\) and if
\begin{equation}
\label{eq:main-regularity}
    \varepsilon_{\mathrm{flat}}
    +\frac{\gamma_N\beta_S}{b_S^-}
    +
    \left(
        \frac{R_Xb_S^++R_\xi\sigma}{\min\{b_S^-,\sigma\}}
    \right)^2
    \sqrt{\frac{r_S+\log(n/\delta)}{n}}
    \le c,
\end{equation}
then, with probability at least \(1-\delta\), every minimum-norm interpolator satisfies the following:
\begin{align*}
    \|\widehat{\bw}_S-\bw_S^\star\|_2
    &<r_\delta,
    \\
    \Rcal(\widehat{\bw})
    &\le
    (b_S^+)^2r_\delta^2
    +
    C\beta_{N,\Sigma}^2
    \left(
        \frac{\sigma}{\gamma_N}
        +\|(\bw_S^\star,\zero)\|
        +\beta_Sr_\delta
    \right)^2
    +2\nu_N^2 .
\end{align*}
\end{theorem}

Once again, the conditions of the theorem that ensure benign overfitting are best highlighted when making the results asymptotic, as given by the following corollary:

\begin{corollary}[Benign overfitting for general norms]
\label{cor:general-asymptotic}
Consider a sequence of problem instances indexed by the sample size \(n\), each satisfying the assumptions of \thmref{thm:main}. Write
\(\gamma_N:=\kappa_N/\sqrt n\), and suppose that \(R_X,R_\xi,\sigma,\sigma^{-1},(b_S^-)^{-1},b_S^+, A^{-1}\) and \(\|\bw_S^\star\|_2\) are uniformly bounded.
If there is a sequence \(\delta_n\to0\) such that, with
\(\varepsilon_{\mathrm{flat}}\) evaluated at \(\delta=\delta_n\),
\[
    \varepsilon_{\mathrm{flat}}\longrightarrow 0,
    \qquad
    \beta_S\gamma_N \longrightarrow 0,
    \qquad
    \frac{\beta_{N,\Sigma}}{\gamma_N}\longrightarrow 0,
\]
then, with probability tending to one, every minimum-norm interpolator satisfies
\[
    \Rcal(\widehat\bw)\longrightarrow0.
\]
\end{corollary}

For reference, in the setting of \thmref{thm:lp-clean}, the conditions $\beta_S\gamma_N \to 0$ and $\frac{\beta_{N,\Sigma}}{\gamma_N}\to 0$
are equivalent to:
\begin{align}\label{eq:conditions-general}
    s_N^2 |N|^{2/q}|S|^{\max\{2/p-1, 0\}}
    &\ll n \ll |N|^{\min\{2/q,1\}},
\end{align}
where $\varepsilon_{\mathrm{flat}}$ depends on the $\varepsilon$ values from Assumptions \ref{assump:K} and \ref{assump:xnl_concentrate}. 
In the setting of \thmref{thm:lp-clean}, the second inequality in \eqref{eq:conditions-general} and the condition $n\gg r_S$ together imply $\varepsilon_{\mathrm{flat}} \to 0$ and thus benign overfitting.

\subsection*{AI use statement}
The fundamental proof ideas in the paper were developed by the authors, without the use of AI. In particular, we did not use AI tools to develop the critical proof ideas and the mechanism underlying our results (described in \secref{sec:mechanism}), as well as early versions of the proofs. Nevertheless, we did use AI tools for finding references and known results that were needed for our proofs, sharpening proofs and weakening their assumptions, writing the code for the experiments, and general editing/writing.
We have reviewed all AI-assisted work, including everything mentioned in the previous sentence. We take responsibility for the final content of this work, including text, claims or artifacts produced with the aid of generative AI.

\newpage

\bibliography{bib}
\bibliographystyle{iclr/iclr2027_conference}

\newpage

\tableofcontents

\newpage

\appendix
\begin{figure*}[t]
    \centering
    \includegraphics[width=\textwidth]{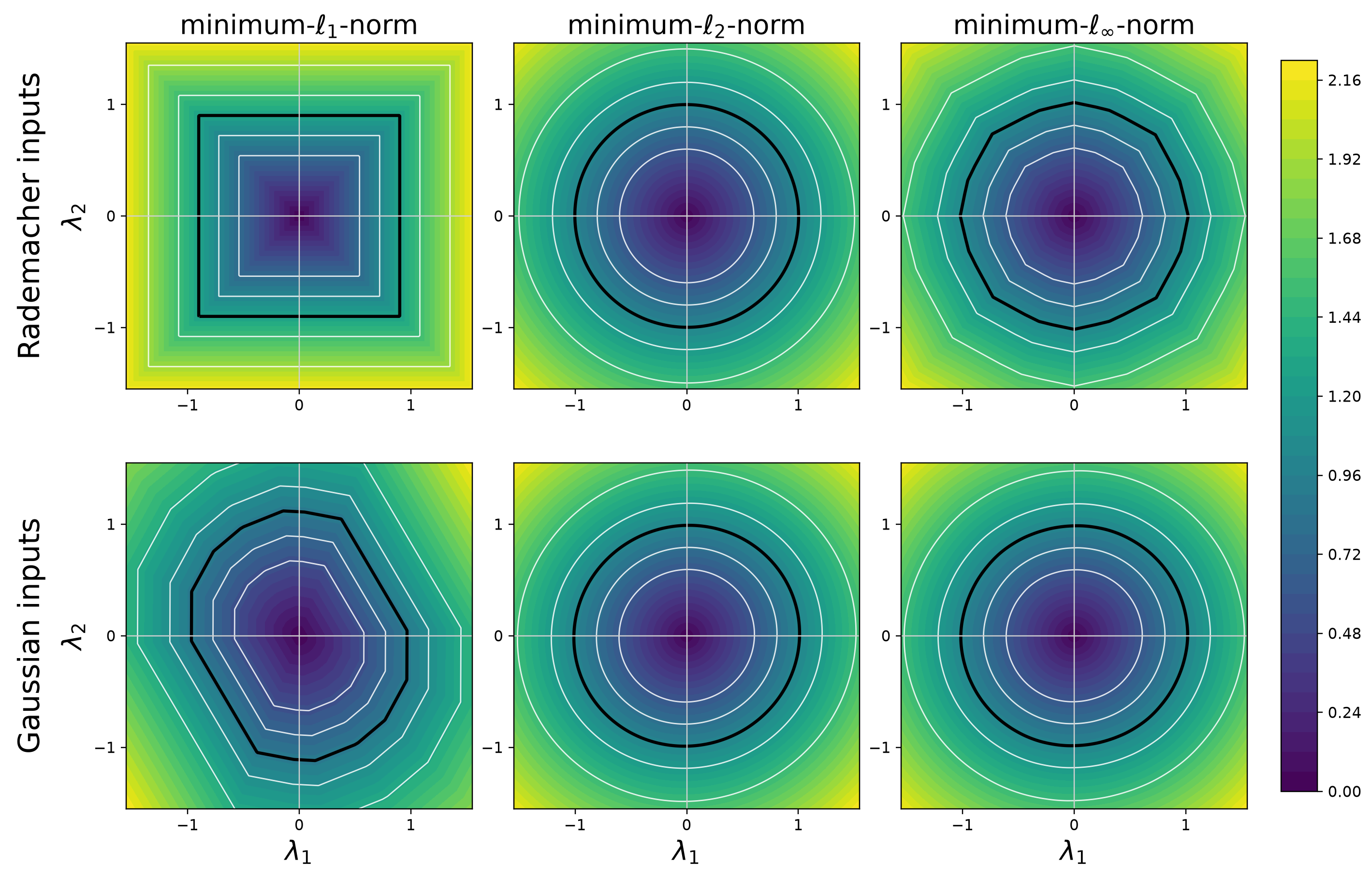}
    \caption{
    Dual geometry for Rademacher (top) and Gaussian (bottom) inputs with $n=12$ and $|N|=2^n$.
    Each panel shows normalized level sets of $\blambda \mapsto \|X_N^\top \blambda\|_q$ on a random flat two-dimensional subspace $\blambda = a \bu + b \bv$, where $\bu,\bv\in\{\pm 1/\sqrt n\}^n$ are orthonormal and $q$ is dual to the indicated minimum-$\ell_p$ norm.
    The black curve is the unit level set.
    The geometries differ substantially: For Gaussian inputs, the geometry is nearly Euclidean across all three norms. However, for Rademacher inputs, the geometry varies substantially between the norms.
    }
    \label{fig:dual-geometry}
\end{figure*}

\section{Related Work}
\label{sec:related-work}

\paragraph{Minimum-$\ell_2$-norm interpolation.}
There is by now a very large body of work studying the generalization error of minimum-$\ell_2$-norm interpolation in linear regression \citep{kobak2020optimal, bartlett2020benign, muthukumar2020harmless, belkin2020two, derezinski2020exact, negrea2020defense, richards2021asymptotics, muthukumar2021classification,hastie2022surprises, bach2023high, shamir2023implicit,tsigler2023benign, mallinar2024}, as well as random-feature/kernel regression \citep{liang2020just, bordelon2020spectrum, cui2021generalization, xiao2022precise, mei2022generalization, lu2023optimal, cheng2024comprehensive, barzilaigeneralization}. In many cases, sharp excess risk bounds are possible when analyzing the minimum-$\ell_2$-norm interpolator under relatively mild assumptions. However, the overwhelming majority of these analyses \emph{begin} with the known closed-form formula for the minimum $\ell_2$-norm interpolator, namely $\hat{\bw}=X^\top(XX^\top)^{\dagger}\by$, which is highly specific to the $\ell_2$ case. As such, many of the techniques and tools developed in such works do not extend to the analysis of minimum-norm interpolation under more general norms. 

\paragraph{General norms.}
Much less is known about minimum-norm interpolation for non-$\ell_2$ norms, and existing works largely rely on Gaussian inputs (and often Gaussian noise as well) \citep{ju2020overfitting, koehler2021uniform, chatterji2022foolish, donhauser2022fast, wang2022tight, kur2026gaussianity}. In most of these works, Gaussianity is crucially utilized by relying on the Convex Gaussian Minmax theorem \citep{thrampoulidis2015regularized}, which does not generalize to other distributions. Under such assumptions, perhaps the best studied norm is $\ell_1$ (also known as basis pursuit), where for isotropic Gaussian inputs with variance $\sigma^2$, sharp rates scaling as $\sigma^2/\log(d/n)$ are known. While this converges to $0$ when $d/n \to \infty$, we show that Gaussianity is critical for this bound, and provide an example of excess risk lower bounded as $\sigma^2 \log(n)$ for sub-Gaussian inputs and noise. Within these works, perhaps the closest to our setting is \citet{koehler2021uniform}, who analyze general norms and split the inputs into a low-dimensional block for fitting the target and a high-dimensional block for fitting the noise, as we do here. Nevertheless, the Gaussian Minmax theorem is central to their results, preventing them from analyzing more general distributions. 

\citet{kur2024banach,kur2026minimum} provide sharp excess risk bounds for $\ell_p$ norms with $p \in (1, 2]$ that only assume sub-Gaussian inputs. Nevertheless, they assume Gaussian noise and isotropic inputs, and their risk bounds do not apply to more general norms such as $\ell_p$ with $p>2$. 
\citet{chinot2022robustness} analyze more general input/noise distributions, allowing even for adversarial noise. In the case where the label noise is independent with variance $\sigma^2$, they establish risk bounds of the form $\bigO\left(\sigma^2\right)$, but these do not imply benign overfitting, meaning that the excess risk does not tend to $0$ when $\sigma^2 > 0$.

\section{Auxiliary Lemmas}
\label{sec:auxiliary-lemmas}

In this section we collect standard concentration results for later use. 
Recall that, for nonempty \(T\subseteq[d]\),
\[
    \Sigma_T
    :=
    \E_{\bx\sim\Dcal_{\bx}}[\bx_T\bx_T^\top]
    =
    \operatorname{diag}\bigl((b_j^2)_{j\in T}\bigr).
\]
The last identity follows from
Assumption~\ref{assump:subgaussian-input-distribution}: the coordinates are
independent and mean-zero, with variances \(b_j^2\).
For a nonempty coordinate set \(T\subseteq[d]\) and
\(\delta\in(0,1)\), define the quadratic concentration scale
\begin{equation}
\label{eq:quadratic-concentration-scale}
    \Delta_2(T,\delta)
    :=
    \sqrt{
        \frac{r_T+\log(8/\delta)}{n}
    }.
\end{equation}
For the empty set, define instead the scalar Bernstein scale
\[
    \Delta_2(\emptyset,\delta)
    :=
    \left[
        \sqrt{\frac{\log(8/\delta)}{n}}
        +
        \frac{\log(8/\delta)}{n}
    \right].
\]
For nonempty \(T\), also define the maximum-coordinate scale
\begin{equation}
\label{eq:max-concentration-scale}
    \Delta_\infty(T,\delta)
    :=
    \frac{1}{\sqrt n}
    \left[
        R_X\|\bb_T\|_2
        +
        \left(R_Xb_T^++R_\xi\sigma\right)
        \sqrt{\log\frac{8n}{\delta}}
    \right].
\end{equation}

\begin{lemma}
\label{lem:anisotropic-euclidean-norm}
Let
\(Z=(Z_1,\ldots,Z_m)\) have independent, mean-zero, sub-Gaussian coordinates, and write \(K_j:=\|Z_j\|_{\psi_2}\). There exists an absolute constant $C>0$, such that for every \(t>0\),
with probability at least \(1-2e^{-t}\),
\begin{equation}
\label{eq:anisotropic-squared-norm-concentration}
    \left|
        \|Z\|_2^2-\sum_{j=1}^m\E Z_j^2
    \right|
    \le
    C\left(
        \sqrt{t\sum_{j=1}^mK_j^4}
        +t\max_{j\in[m]}K_j^2
    \right),
\end{equation}
and
\begin{equation}
\label{eq:anisotropic-norm-concentration}
    \|Z\|_2
    \le
    C\left[
        \left(\sum_{j=1}^mK_j^2\right)^{1/2}
        +
        \max_{j\in[m]}K_j\sqrt t
    \right].
\end{equation}
\end{lemma}

\begin{proof}
Throughout this proof, we use $C$ to denote an absolute constant that may change from line to line.
The variables \(Z_j^2-\E Z_j^2\) are independent, mean-zero, and, by
\cite[Lemmas 2.7.8, 2.8.5]{Vershynin_2026}, have \(\psi_1\)-norms at most
\(C K_j^2\). Applying Bernstein's inequality
\cite[Theorem 2.9.1]{Vershynin_2026} gives
\eqref{eq:anisotropic-squared-norm-concentration}. Moreover,
\[
    \sum_{j=1}^m\E Z_j^2
    \le
    C\sum_{j=1}^mK_j^2,
    \qquad
    \left(\sum_{j=1}^mK_j^4\right)^{1/2}
    \le
    \max_{j\in[m]}K_j
    \left(\sum_{j=1}^mK_j^2\right)^{1/2}.
\]
Thus, squaring the right-hand side of
\eqref{eq:anisotropic-norm-concentration} shows it dominates the upper bound for \(\|Z\|_2^2\) given by
\eqref{eq:anisotropic-squared-norm-concentration}.
\end{proof}

\begin{lemma}
\label{lem:basic-concentration-consequences}
Suppose Assumptions~\ref{assump:subgaussian-input-distribution}
and~\ref{assump:noise-variables} hold. Then there exists an absolute constant $C>0$ such that for every nonempty
\(T\subseteq[d]\), every \(\delta\in(0,1)\), and every 
$n\ge C\left(r_T+\log(8/\delta)\right)$,
with probability at least \(1-\delta\), the following bounds hold simultaneously:
\[
    \left\|
        \frac1nX_T^\top X_T-\Sigma_T
    \right\|_{\op}
    \le
    C R_X^2(b_T^+)^2\Delta_2(T,\delta),
\]
\[
    \frac1{\sqrt n}
    \left(
        \max_{i\in[n]}\|\bx_{i,T}\|_2
        +
        \|\bxi\|_\infty
    \right)
    \le
    C\Delta_\infty(T,\delta),
\]
\[
    \left\|
        \frac1nX_T^\top\bxi
    \right\|_2
    \le
    C R_XR_\xi\sigma b_T^+\Delta_2(T,\delta),
\]
\[
    \left|
        \frac1n\|\bxi\|_2^2-\sigma^2
    \right|
    \le
    C R_\xi^2\sigma^2\Delta_2(\emptyset,\delta).
\]
Moreover,
\[
    \frac1{\sqrt n}\|X_N\bw_N^\star\|_2
    \le
    \left(
        1
        +
        C R_X^2\Delta_2(\emptyset,\delta)
    \right)^{1/2}
    \nu_N.
\]
\end{lemma}

\begin{proof}
Throughout this proof, we use $C$ to denote an absolute constant that may change from line to line.
For every \(\bu\in\R^{T}\), independence and centering of the coordinates,
together with the sub-Gaussian linear-combination inequality
\cite[Proposition 2.7.1]{Vershynin_2026}, give
\[
    \|\langle\bu,\bx_T\rangle\|_{\psi_2}
    \le
    C\left(\sum_{j\in T}u_j^2\|x_j\|_{\psi_2}^2\right)^{1/2}
    \le
    C R_X\left(\bu^\top\Sigma_T\bu\right)^{1/2}.
\]
Thus condition~(4) of Example~1 in
\citet{zhivotovskiy2024dimension} holds with \(\kappa=C R_X\), and hence
condition~(2) of Theorem~1 holds for
\(M=\bx_T\bx_T^\top\). Applying that theorem with
\(t=\log(8/\delta)\), using
\[
    \frac{\operatorname{tr}(\Sigma_T)}{\|\Sigma_T\|_{\op}}
    =r_T,
    \qquad
    \|\Sigma_T\|_{\op}=(b_T^+)^2,
\]
and the assumed sample-size condition, gives, with probability at least
\(1-\delta/8\),
\[
    \left\|
        \frac1nX_T^\top X_T-\Sigma_T
    \right\|_{\op}
    \le
    C R_X^2(b_T^+)^2
    \sqrt{\frac{r_T+\log(8/\delta)}{n}}.
\]

For the row norm, apply
Lemma~\ref{lem:anisotropic-euclidean-norm} with
\(t=\log(16n/\delta)\) to each \(\bx_{i,T}\), using
\(\|x_{i,j}\|_{\psi_2}\le R_Xb_j\), and take a union bound over
\(i\in[n]\). This gives, with probability at least \(1-\delta/8\), after
adjusting the absolute constant,
\[
    \frac1{\sqrt n}\max_{i\in[n]}\|\bx_{i,T}\|_2
    \le
    \frac{C R_X}{\sqrt n}
    \left(
        \|\bb_T\|_2
        +
        b_T^+\sqrt{\log(8n/\delta)}
    \right).
\]

For the cross term, condition on \(\bxi\), and write
\[
    Y_j:=\frac{1}{n}\sum_{i=1}^n x_{i,j}\xi_i,
    \qquad j\in T.
\]
The variables \((Y_j)_{j\in T}\) are then independent and mean-zero, and
the sub-Gaussian linear-combination inequality gives
\[
    \|Y_j\|_{\psi_2}
    \le
    \frac{C R_Xb_j}{n}\|\bxi\|_2 .
\]
Conditionally applying Lemma~\ref{lem:anisotropic-euclidean-norm} with
\(t=\log(16/\delta)\) therefore gives, with conditional probability at least
\(1-\delta/8\),
\[
    \left\|\frac1nX_T^\top\bxi\right\|_2
    \le
    \frac{C R_X\|\bxi\|_2}{n}
    \left(
        \|\bb_T\|_2
        +b_T^+\sqrt{\log(8/\delta)}
    \right)
\]
after adjusting the absolute constant. The same lemma, with the same value of
\(t\), applied to \(\bxi\) gives
\[
    \|\bxi\|_2
    \le
    C R_\xi\sigma\left(\sqrt n+\sqrt{\log(8/\delta)}\right)
\]
with probability at least \(1-\delta/8\). 
Under the assumed sample-size condition, \(\log(8/\delta)\le n\), and the last two displays, together with
the exact identity
\[
    \|\bb_T\|_2=b_T^+\sqrt{r_T},
\]
give
\[
    \left\|\frac1nX_T^\top\bxi\right\|_2
    \le
    C R_XR_\xi\sigma b_T^+
    \sqrt{\frac{r_T+\log(8/\delta)}{n}}.
\]

The \(\ell_\infty\) noise bound follows from the standard sub-Gaussian tail
bound, e.g. \cite[Proposition 2.6.1]{Vershynin_2026}, and a union bound
over \(i\in[n]\). Thus, with probability at least \(1-\delta/8\),
\[
    \frac1{\sqrt n}\|\bxi\|_\infty
    \le
    C R_\xi\sigma\sqrt{\frac{\log(8n/\delta)}{n}}.
\]
Adding this to the row-norm bound proves the stated
\(C\Delta_\infty(T,\delta)\) bound.

For the noise energy, apply
\eqref{eq:anisotropic-squared-norm-concentration} with
\(t=\log(16/\delta)\) to \(\bxi\), using
\(\E\xi_i^2=\sigma^2\) and
\(\|\xi_i\|_{\psi_2}\le R_\xi\sigma\). With probability at least
\(1-\delta/8\), this gives
\[
    \left|
        \frac1n\|\bxi\|_2^2-\sigma^2
    \right|
    \le
    C R_\xi^2\sigma^2\Delta_2(\emptyset,\delta).
\]

Finally, set \(V_i:=\bx_{i,N}^\top\bw_N^\star\). These variables are
independent and mean-zero, with
\[
    \E V_i^2=\nu_N^2,
    \qquad
    \|V_i\|_{\psi_2}\le C R_X\nu_N.
\]
Another application of
\eqref{eq:anisotropic-squared-norm-concentration}, again with
\(t=\log(16/\delta)\), gives, with probability at least \(1-\delta/8\),
\[
    \left|
        \frac1n\|X_N\bw_N^\star\|_2^2-\nu_N^2
    \right|
    \le
    C R_X^2\Delta_2(\emptyset,\delta)\nu_N^2.
\]
Taking the upper bound and then the square root proves the last conclusion.
Taking a union bound over the above events completes the proof.
\end{proof}

\section{Proof of Theorem~\ref{thm:main} and Related Lemmas}
\label{sec:general-results}

\subsection{Approximate Isometry on Flat Vectors}
Let 
\[
    K(\blambda)
    :=\frac{\E\|X_N^\top\blambda\|_{*}}{\kappa_N}.
\]
As a reminder, the norm \(\|\cdot\|_*\) in the definition of $K(\blambda)$ is defined as the dual norm of the induced $N$-block norm, i.e. \(\|\bv\|_{*} = \sup_{\|(\zero,\bu)\|\le1}\bv^\top\bu\).
\begin{lemma}[Dual truncation]
\label{lem:truncate}
Suppose Assumptions~\ref{assump:subgaussian-input-distribution}
and~\ref{assump:K} hold. For \(\tau>0\), define the truncated vector
\[
    \blambda_\tau
    :=\sum_{j=1}^n
    \one_{\{|\lambda_j|\le\tau\}}\lambda_j\be_j.
\]
Then every \(\blambda\) with \(K(\blambda)\le1\) satisfies for every \(\bz\in\R^n\):
\[
    K(\blambda_\tau)\le K(\blambda),
    \qquad
    \bz^\top\blambda_\tau
    \ge
    \bz^\top\blambda-
    \frac{\|\bz\|_\infty}{A^2\tau}.
\]
\end{lemma}

\begin{proof}
If \(\blambda=0\), the result is trivial. 
Otherwise, the assumed lower bound and homogeneity imply
\[
    K(\blambda)
    =
    \|\blambda\|_2K\left(\frac{\blambda}{\|\blambda\|_2}\right)
    \ge
    A\|\blambda\|_2,
\]
so every feasible \(\blambda\) satisfies
\begin{equation}
\label{eq:feasible-lambda-norm}
    \|\blambda\|_2\le A^{-1}.
\end{equation}

Fix a coordinate \(i\), and let
\(\blambda'=\blambda-\lambda_i\be_i\). Since \(\bx_{i,N}\) is independent of the other rows and has mean zero, Jensen's inequality gives
\[
    \kappa_N K(\blambda)
    =\E\left\|\sum_{j=1}^n\lambda_j\bx_{j,N}\right\|_*
    \ge
    \E\left\|\sum_{j\ne i}\lambda_j\bx_{j,N}
    +\lambda_i\E[\bx_{i,N}]\right\|_*
    =\kappa_N K(\blambda').
\]
Sequentially applying this zeroing operation over all coordinates with
\(|\lambda_j|>\tau\) gives \(K(\blambda_\tau)\le K(\blambda)\). Let \(\Ical_\tau:=\{j:|\lambda_j|>\tau\}\). Since
\[
    \blambda-\blambda_\tau
    =
    \sum_{j\in\Ical_\tau}\lambda_j\be_j,
\]
equation~\eqref{eq:feasible-lambda-norm} gives
\[
    \sum_{j\in\Ical_\tau}|\lambda_j|
    \le
    \frac1{\tau}\sum_{j\in\Ical_\tau}\lambda_j^2
    \le
    \frac{\|\blambda\|_2^2}{\tau}
    \le
    \frac1{A^2\tau}.
\]
Therefore,
\[
    \bz^\top\blambda_\tau
    =
    \bz^\top\blambda
    -
    \sum_{j\in\Ical_\tau}z_j\lambda_j
    \ge
    \bz^\top\blambda
    -
    \|\bz\|_\infty
    \sum_{j\in\Ical_\tau}|\lambda_j|
    \ge
    \bz^\top\blambda-\frac{\|\bz\|_\infty}{A^2\tau}.
\]
\end{proof}

\begin{proposition}[Flat-vector dual approximation]
\label{prop:denseCLT}
Suppose Assumptions~\ref{assump:subgaussian-input-distribution}
and~\ref{assump:K} hold. Let \(\bz\in\R^n\) be nonzero, and
write
$
    \varepsilon_{\mathrm{geom}}
    :=
    \varphi\left(
        2\sqrt{\frac{\|\bz\|_\infty}{\|\bz\|_2}}
    \right).
$
If \(\varepsilon_{\mathrm{geom}}\le1/2\), then
\[
    (1-\varepsilon_{\mathrm{geom}})\|\bz\|_2
    \le
    \max_{\blambda:K(\blambda)\le1}\bz^\top\blambda
    \le
    \left(
        1+2\varepsilon_{\mathrm{geom}}
        +
        A^{-2}\sqrt{\frac{\|\bz\|_\infty}{\|\bz\|_2}}
    \right)\|\bz\|_2 .
\]
\end{proposition}

\begin{proof}
First, note that the feasible set  $\{\blambda : K(\blambda) \le 1\}$ is compact: it is closed, and
Assumption~\ref{assump:K} gives \(K(\blambda)\ge A\|\blambda\|_2\) implying it is bounded. Thus the
maximum is attained. Let
\[
    \bar{\bz}:=\frac{\bz}{\|\bz\|_2},
    \qquad
    \tau:=\sqrt{\|\bar{\bz}\|_\infty}
    =
    \sqrt{\frac{\|\bz\|_\infty}{\|\bz\|_2}}.
\]
Since \(\varphi(\cdot)\) is nondecreasing and
\(\|\bar{\bz}\|_\infty\le2\tau\), Assumption~\ref{assump:K} gives
\[
    K(\bar{\bz})
    \le
    1+\varphi(\|\bar{\bz}\|_\infty)
    \le
    1+\varepsilon_{\mathrm{geom}} .
\]
Hence \(\bar{\bz}/(1+\varepsilon_{\mathrm{geom}})\) is feasible, and therefore
\[
    \max_{\blambda:K(\blambda)\le1}\bz^\top\blambda
    \ge
    \frac{\|\bz\|_2}{1+\varepsilon_{\mathrm{geom}}}
    \ge
    (1-\varepsilon_{\mathrm{geom}})\|\bz\|_2.
\]

For the upper bound, let \(\blambda\) be any feasible point. By
Lemma~\ref{lem:truncate},
\[
    \bz^\top\blambda
    \le
    \bz^\top\blambda_\tau+\frac{\|\bz\|_\infty}{A^2\tau}
    \le
    \|\bz\|_2\|\blambda_\tau\|_2
    +
    \frac{\|\bz\|_\infty}{A^2\tau}.
\]
It remains to bound \(\|\blambda_\tau\|_2\). If
\(\|\blambda_\tau\|_2<1/2\), then
\[
    \|\blambda_\tau\|_2
    \le
    \frac1{1-\varepsilon_{\mathrm{geom}}}.
\]
Otherwise, \(\|\blambda_\tau\|_2\ge1/2\), and since every coordinate of
\(\blambda_\tau\) has magnitude at most \(\tau\),
\[
    \frac{\|\blambda_\tau\|_\infty}{\|\blambda_\tau\|_2}
    \le
    2\tau.
\]
Using \(K(\blambda_\tau)\le1\), homogeneity, and Assumption~\ref{assump:K},
\[
    1
    \ge
    K(\blambda_\tau)
    =
    \|\blambda_\tau\|_2
    K\left(\frac{\blambda_\tau}{\|\blambda_\tau\|_2}\right)
    \ge
    \|\blambda_\tau\|_2(1-\varepsilon_{\mathrm{geom}}).
\]
Thus, in all cases,
\[
    \|\blambda_\tau\|_2
    \le
    \frac1{1-\varepsilon_{\mathrm{geom}}}.
\]
Consequently,
\[
    \bz^\top\blambda
    \le
    \|\bz\|_2
    \left(
        \frac1{1-\varepsilon_{\mathrm{geom}}}
        +
        \frac{\|\bz\|_\infty}{A^2\tau\|\bz\|_2}
    \right).
\]
Since \(\tau=\sqrt{\|\bz\|_\infty/\|\bz\|_2}\) and
\((1-\varepsilon_{\mathrm{geom}})^{-1}
\le1+2\varepsilon_{\mathrm{geom}}\), this becomes
\[
    \bz^\top\blambda
    \le
    \left(
        1+2\varepsilon_{\mathrm{geom}}
        +
        A^{-2}\sqrt{\frac{\|\bz\|_\infty}{\|\bz\|_2}}
    \right)\|\bz\|_2 .
\]
Taking the maximum over feasible \(\blambda\) completes the proof.
\end{proof}

\begin{corollary}[Random flat-vector dual approximation]
\label{cor:denseCLT}
Suppose Assumptions~\ref{assump:subgaussian-input-distribution},
\ref{assump:K}, and~\ref{assump:xnl_concentrate} hold.
Fix \(\delta\in(0,1)\) such that
\(\varepsilon_{\mathrm{unif}}:=\varepsilon_{\mathrm{unif}}(\delta)<1/2\). Then, with probability at least
\(1-\delta\), the following holds for every nonzero \(\bz\in\R^n\). Define
$
    \varepsilon_{\mathrm{geom}}
    :=
    \varphi\left(
        2\sqrt{\frac{\|\bz\|_\infty}{\|\bz\|_2}}
    \right),
$
and suppose that \(\varepsilon_{\mathrm{geom}}\le1/2\). Then
\[
\begin{aligned}
    \frac{(1-\varepsilon_{\mathrm{unif}})(1-\varepsilon_{\mathrm{geom}})}{\kappa_N}\|\bz\|_2
    &\le
    \max_{\blambda:\|X_N^\top\blambda\|_{*}\le1}\bz^\top\blambda
    \\
    &\le
    \frac{(1+2\varepsilon_{\mathrm{unif}})
    \left(
        1+2\varepsilon_{\mathrm{geom}}
        +
        A^{-2}\sqrt{\frac{\|\bz\|_\infty}{\|\bz\|_2}}
    \right)}
    {\kappa_N}
    \|\bz\|_2 .
\end{aligned}
\]
\end{corollary}

\begin{proof}
By Assumption~\ref{assump:xnl_concentrate}, with probability at least
\(1-\delta\), the defining inequality in
Assumption~\ref{assump:xnl_concentrate} holds. For \(\blambda=0\) the next
inequality is trivial, while for \(\blambda\ne0\) it follows by applying
Assumption~\ref{assump:xnl_concentrate} to
\(\blambda/\|\blambda\|_2\). Thus, for every \(\blambda\in\R^n\),
\[
    (1-\varepsilon_{\mathrm{unif}})\kappa_NK(\blambda)
    \le
    \|X_N^\top\blambda\|_{*}
    \le
    (1+\varepsilon_{\mathrm{unif}})\kappa_NK(\blambda).
\]
Therefore,
\[
    \left\{\blambda:K(\blambda)\le \frac1{(1+\varepsilon_{\mathrm{unif}})\kappa_N}\right\}
    \subseteq
    \left\{\blambda:\|X_N^\top\blambda\|_{*}\le1\right\}
    \subseteq
    \left\{\blambda:K(\blambda)\le \frac1{(1-\varepsilon_{\mathrm{unif}})\kappa_N}\right\}.
\]
By homogeneity of \(K\),
\[
    \frac1{(1+\varepsilon_{\mathrm{unif}})\kappa_N}
    \max_{\blambda:K(\blambda)\le1}\bz^\top\blambda
    \le
    \max_{\blambda:\|X_N^\top\blambda\|_{*}\le1}\bz^\top\blambda
    \le
    \frac1{(1-\varepsilon_{\mathrm{unif}})\kappa_N}
    \max_{\blambda:K(\blambda)\le1}\bz^\top\blambda .
\]
Apply Proposition~\ref{prop:denseCLT}. For the lower bound, use
\((1+\varepsilon_{\mathrm{unif}})^{-1}\ge1-\varepsilon_{\mathrm{unif}}\), and for the upper bound use
\((1-\varepsilon_{\mathrm{unif}})^{-1}\le1+2\varepsilon_{\mathrm{unif}}\), valid since
\(\varepsilon_{\mathrm{unif}}<1/2\).
\end{proof}

\subsection{Uniform Approximation of the Dual Objective}
For
\(\bw_S\in\R^{S}\), define
\[
    \bz(\bw_S)
    :=
    X_S(\bw_S^\star-\bw_S)+\bxi ,
\]
and
\[
    \Wcal_S
    :=
    \left\{
        \bw_S\in\R^{S}:
        \|\bw_S-\bw_S^\star\|_2\le1
    \right\}.
\]
Recall that
\[
    \nu_N^2
    =
    \sum_{j\in N}b_j^2(w_j^\star)^2 .
\]
Thus \(\nu_N^2\) is the population risk contribution of the true
\(N\)-component.

On the event where the constraint is feasible, define
\[
    g(\bw_S)
    :=
    \min_{\bw_N:\,X_N\bw_N=\bz(\bw_S)+X_N\bw_N^\star}
    \|(\zero,\bw_N)\|.
\]

\begin{lemma}[Dual representation]
\label{lem:dual}
Fix a matrix \(X_N\) and $\br \in \R^n$, and suppose the constraint
\(X_N\bw_N=\br\) is feasible. Then
\[
    \min_{\bw_N:\,X_N\bw_N=\br}\|(\zero,\bw_N)\|
    =
    \max_{\blambda\in\R^n}
    \left\{
        \br^\top\blambda:
        \|X_N^\top\blambda\|_{*}\le 1
    \right\},
\]
where \(\|\cdot\|_{*}\) is the dual norm associated with
\(\bw_N\mapsto\|(\zero,\bw_N)\|\).
\end{lemma}

\begin{proof}
The Lagrangian is
\[
    L(\bw_N,\blambda)
    =
    \|(\zero,\bw_N)\|
    +\blambda^\top(\br-X_N\bw_N).
\]
Hence the dual function is
\[
\inf_{\bw_N}L(\bw_N,\blambda)
=
\br^\top\blambda
+
\inf_{\bw_N}
\left\{
    \|(\zero,\bw_N)\|
    -(X_N^\top\blambda)^\top\bw_N
\right\}.
\]
By the definition of the dual norm, the infimum in braces is \(0\)
when \(\|X_N^\top\blambda\|_*\le1\), attained at \(\bw_N=\zero\),
and is \(-\infty\) otherwise. Thus the Lagrange dual is
\[
    \max_{\blambda\in\R^n}
    \left\{
        \br^\top\blambda:
        \|X_N^\top\blambda\|_*\le1
    \right\}.
\]
Finally, the primal objective is finite and continuous everywhere, and
the affine constraint is feasible by assumption. Therefore, Slater's
condition holds, so strong duality holds and the dual optimum is attained.
\end{proof}

Fix an absolute constant \(C_{\mathrm{conc}}>0\) large enough to dominate the
constant in both the sample-size condition and every conclusion of
Lemma~\ref{lem:basic-concentration-consequences}. For
\(\delta\in(0,1)\), define
\begin{align}
\label{eq:g-bound-error-parameters}
    \eta_N(\delta)
    &:=
    \left(
        1
        +
        C_{\mathrm{conc}} R_X^2\Delta_2(\emptyset,\delta/3)
    \right)^{1/2}\nu_N, \notag\\
    L_S(\delta)
    &:=
    \left(
        \sigma^2
        -
        C_{\mathrm{conc}} R_\xi^2\sigma^2
        \Delta_2(\emptyset,\delta/3)
        -
        2C_{\mathrm{conc}} R_XR_\xi\sigma b_S^+
        \Delta_2(S,\delta/3)
    \right)^{1/2}, \notag\\
    U_S(\delta)
    &:=
    C_{\mathrm{conc}}\Delta_\infty(S,\delta/3), \notag\\
    \rho_{S,N}(\delta)
    &:=
    \frac{
        U_S(\delta)+\eta_N(\delta)
    }{
        L_S(\delta)-\eta_N(\delta)
    } .
\end{align}
These quantities are used only when the right-hand sides are well-defined and
\(L_S(\delta)>\eta_N(\delta)\).

\begin{lemma}[Uniform approximation of the \(N\)-block fitting norm]
\label{lem:g-bound}
Suppose Assumptions~\ref{assump:subgaussian-input-distribution},
\ref{assump:noise-variables}, \ref{assump:K}, and
\ref{assump:xnl_concentrate} hold. Fix \(\delta\in(0,1)\), and let
\(\varepsilon_{\mathrm{unif}}:=\varepsilon_{\mathrm{unif}}(\delta/3)\) be the quantity from
Assumption~\ref{assump:xnl_concentrate}. Let \(\varphi(\cdot)\) and \(A\) be
the quantities from Assumption~\ref{assump:K}, and define
\begin{equation}
\label{eq:epsilon-g-definition}
    \varepsilon_{\mathrm{fit}}(\delta)
    :=
    C
    \left(
        \rho_{S,N}(\delta)
        +
        \varepsilon_{\mathrm{unif}}
        +
        \varphi\left(2\sqrt{\rho_{S,N}(\delta)}\right)
        +
        \sqrt{\rho_{S,N}(\delta)}
    \right),
\end{equation}
where \(C=C(A)\ge1\) is a sufficiently large fixed constant, chosen in
particular so that the sample-size condition below implies the hypothesis of
Lemma~\ref{lem:basic-concentration-consequences} with failure probability
\(\delta/3\). Suppose that
\[
    n\ge C\left(r_S+\log(8/\delta)\right), 
    \qquad 
    L_S(\delta)>\eta_N(\delta),
    \qquad
    \varepsilon_{\mathrm{unif}}<\frac12,
    \qquad
    \varphi\left(2\sqrt{\rho_{S,N}(\delta)}\right)\le\frac12 .
\]
Then, with probability at least
\(1-\delta\), \(X_N\) has full row rank and, simultaneously for every
\(\bw_S\in\Wcal_S\),
\[
    (1-\varepsilon_{\mathrm{fit}}(\delta))\|\bz(\bw_S)\|_2
    \le
    \kappa_N g(\bw_S)
    \le
    (1+\varepsilon_{\mathrm{fit}}(\delta))\|\bz(\bw_S)\|_2 .
\]
\end{lemma}

\begin{proof}
First, by Assumption~\ref{assump:xnl_concentrate}, with probability at least
\(1-\delta/3\),
\begin{equation}
\label{eq:uniform-dual-concentration-event}
    (1-\varepsilon_{\mathrm{unif}})\kappa_NK(\blambda)
    \le
    \|X_N^\top\blambda\|_{*}
    \le
    (1+\varepsilon_{\mathrm{unif}})\kappa_NK(\blambda)
    \qquad
    \text{for all } \blambda\in\R^n .
\end{equation}
Second, by Lemma~\ref{lem:basic-concentration-consequences}, applied with
\(T=S\) and failure probability \(\delta/3\), we have
\begin{equation}
\label{eq:signal-noise-concentration-event}
    \frac1{\sqrt n}
    \left(
        \max_{i\in[n]}\|\bx_{i,S}\|_2
        +
        \|\bxi\|_\infty
    \right)
    \le
    C_{\mathrm{conc}}\Delta_\infty(S,\delta/3)
    =
    U_S(\delta),
\end{equation}
and
\begin{equation}
\label{eq:signal-residual-l2-event}
    \left\|
        \frac1n X_S^\top\bxi
    \right\|_2
    \le
    C_{\mathrm{conc}} R_XR_\xi\sigma b_S^+\Delta_2(S,\delta/3),
    \qquad
    \left|
        \frac1n\|\bxi\|_2^2-\sigma^2
    \right|
    \le
    C_{\mathrm{conc}} R_\xi^2\sigma^2
    \Delta_2(\emptyset,\delta/3).
\end{equation}
The same application of Lemma~\ref{lem:basic-concentration-consequences}
also gives
\begin{equation}
\label{eq:true-N-empirical-event}
    \frac1{\sqrt n}\|X_N\bw_N^\star\|_2
    \le
    \eta_N(\delta).
\end{equation}
By a union bound, the uniform dual-concentration event and the basic
concentration event hold simultaneously with probability at least
\(1-2\delta/3\), and hence at least \(1-\delta\).

On the event \eqref{eq:uniform-dual-concentration-event}, \(X_N\) has full row
rank. Indeed, by
\eqref{eq:uniform-dual-concentration-event} and Assumption~\ref{assump:K},
\[
    \|X_N^\top\blambda\|_{*}
    \ge
    (1-\varepsilon_{\mathrm{unif}})\kappa_NK(\blambda)
    \ge
    (1-\varepsilon_{\mathrm{unif}})\kappa_NA\|\blambda\|_2 .
\]
Since \(\varepsilon_{\mathrm{unif}}<1\), no nonzero vector can lie in the kernel of \(X_N^\top\).
Fix \(\bw_S\in\Wcal_S\), and write
\[
    \bu:=\bw_S^\star-\bw_S,
    \qquad
    \br(\bw_S):=\bz(\bw_S)+X_N\bw_N^\star .
\]
Since \(X_N\) has full row rank, Lemma~\ref{lem:dual} gives
\begin{equation}
\label{eq:g-dual-residual}
    g(\bw_S)
    =
    \max_{\blambda\in\R^n}
    \left\{
        \br(\bw_S)^\top\blambda:
        \|X_N^\top\blambda\|_{*}\le1
    \right\}.
\end{equation}
Using \eqref{eq:signal-noise-concentration-event}, Cauchy--Schwarz, and
\(\|\bu\|_2\le1\),
\begin{equation}
\label{eq:z-infinity-bound}
    \frac1{\sqrt n}\|\bz(\bw_S)\|_\infty
    \le
    \frac1{\sqrt n}
    \left(
        \max_{i\in[n]}\|\bx_{i,S}\|_2 \norm{\bu}_2
        +
        \|\bxi\|_\infty
    \right)
    \leq 
    U_S(\delta).
\end{equation}
Using \eqref{eq:signal-residual-l2-event}, and dropping the nonnegative
quadratic term in \(X_S^\top X_S\),
\begin{equation}
\label{eq:z-l2-lower-bound}
    \frac1{\sqrt n}\|\bz(\bw_S)\|_2
    \ge
    \left(
        \frac1n\|\bxi\|_2^2
        -2\|\bu\|_2
        \left\|\frac1nX_S^\top\bxi\right\|_2
    \right)^{1/2}
    \ge
    L_S(\delta).
\end{equation}
Together with \eqref{eq:true-N-empirical-event}, equations
\eqref{eq:z-infinity-bound} and~\eqref{eq:z-l2-lower-bound} imply
\begin{align}
\label{eq:r-flatness-bound}
    \frac{\|\br(\bw_S)\|_\infty}{\|\br(\bw_S)\|_2}
    \le
    \frac{
        \|\bz(\bw_S)\|_\infty+\|X_N\bw_N^\star\|_2
    }{
        \|\bz(\bw_S)\|_2-\|X_N\bw_N^\star\|_2
    }
    \le
    \frac{U_S(\delta)+\eta_N(\delta)}{
        L_S(\delta)-\eta_N(\delta)
    }
    =
    \rho_{S,N}(\delta).
\end{align}

We next compare the Euclidean norms of \(\br(\bw_S)\) and
\(\bz(\bw_S)\). Their relative difference satisfies
\begin{equation}
\label{eq:r-z-relative-perturbation}
    \frac{\|\br(\bw_S)-\bz(\bw_S)\|_2}{\|\bz(\bw_S)\|_2}
    =
    \frac{\|X_N\bw_N^\star\|_2}{\|\bz(\bw_S)\|_2}
    \le
    \frac{\eta_N(\delta)}{L_S(\delta)}
    \le
    \rho_{S,N}(\delta).
\end{equation}
The triangle and reverse triangle inequalities now
give
\begin{equation}
\label{eq:r-z-norm-comparison}
    (1-\rho_{S,N}(\delta))\|\bz(\bw_S)\|_2
    \le
    \|\br(\bw_S)\|_2
    \le
    (1+\rho_{S,N}(\delta))\|\bz(\bw_S)\|_2 .
\end{equation}

By monotonicity of \(\varphi\) and the assumption on
\(\varphi(2\sqrt{\rho_{S,N}(\delta)})\), the conclusion of
Corollary~\ref{cor:denseCLT} applies on
\eqref{eq:uniform-dual-concentration-event}, uniformly over
\(\bw_S\in\Wcal_S\). Combining that conclusion with
\eqref{eq:g-dual-residual} gives
\begin{align*}
    \kappa_Ng(\bw_S)
    &\ge
    (1-\varepsilon_{\mathrm{unif}})
    \left(
        1-\varphi(2\sqrt{\rho_{S,N}(\delta)})
    \right)
    \|\br(\bw_S)\|_2 \\
    &\ge
    (1-\varepsilon_{\mathrm{unif}})
    \left(
        1-\varphi(2\sqrt{\rho_{S,N}(\delta)})
    \right)
    (1-\rho_{S,N}(\delta))\|\bz(\bw_S)\|_2,
\end{align*}
and
\begin{align*}
    \kappa_Ng(\bw_S)
    &\le
    (1+2\varepsilon_{\mathrm{unif}})
    \left(
        1
        +2\varphi(2\sqrt{\rho_{S,N}(\delta)})
        +A^{-2}\min\left\{\sqrt{\rho_{S,N}(\delta)},1\right\}
    \right)
    \|\br(\bw_S)\|_2 \\
    &\le
    (1+2\varepsilon_{\mathrm{unif}})
    \left(
        1
        +2\varphi(2\sqrt{\rho_{S,N}(\delta)})
        +A^{-2}\min\left\{\sqrt{\rho_{S,N}(\delta)},1\right\}
    \right)\\
    &\quad{}\times
    (1+\rho_{S,N}(\delta))\|\bz(\bw_S)\|_2.
\end{align*}
Here the second inequality in each display uses
\eqref{eq:r-z-norm-comparison}. When \(\rho_{S,N}(\delta)\le1\), expanding
the products and using \(\varepsilon_{\mathrm{unif}}\le1/2\) and
\(\varphi(2\sqrt{\rho_{S,N}(\delta)})\le1/2\) gives the claimed bounds after
choosing the constant in \eqref{eq:epsilon-g-definition} sufficiently large
depending only on \(A\). When \(\rho_{S,N}(\delta)>1\), the lower bound is
trivial because \(\varepsilon_{\mathrm{fit}}(\delta)\ge1\), while the upper bound follows
by absorbing \(1+\rho_{S,N}(\delta)\) into
\eqref{eq:epsilon-g-definition}.
\end{proof}

\subsection{Proof of Theorem~\ref{thm:main}}

\begin{proof}
Throughout this proof, \(C\), \(C'\), and \(c\) denote constants that may depend on \(A\).
Set
\[
\begin{aligned}
    \delta_0&:=\frac\delta2,
    &\qquad
    \varepsilon_{\mathrm{fit}}&:=\varepsilon_{\mathrm{fit}}(\delta_0),
    \qquad 
    \Delta_S&:=\Delta_2(S,\delta_0),
    \qquad
    \Delta_0&:=\Delta_2(\emptyset,\delta_0).
\end{aligned}
\]
Since \(\delta_0=\delta/2\), \(r_S\ge1\), and the theorem's
sample-size condition holds, the definitions of \(\Delta_2\) give
\[
    \Delta_S+\Delta_0
    \le
    3\sqrt{\frac{r_S+\log(n/\delta)}{n}}.
\]
We next verify the hypotheses of Lemma~\ref{lem:g-bound}. Recall that
\(\eta_N\), \(L_S\), \(U_S\), and \(\rho_{S,N}\) are defined in
\eqref{eq:g-bound-error-parameters}, while \(\rho\) and
\(\varepsilon_{\mathrm{flat}}\) are defined in \eqref{eq:main-scales}.
The preceding comparison and \eqref{eq:main-regularity} give
\[
    L_S(\delta_0)>\eta_N(\delta_0),
    \qquad
    \rho_{S,N}(\delta_0)\le C\rho.
\]
Together with the definition of \(\varepsilon_{\mathrm{flat}}\), a
sufficiently large \(C'\), and a sufficiently small \(c\), this verifies all
the hypotheses of Lemma~\ref{lem:g-bound}. We also record various bounds that
will be used below. Using \eqref{eq:epsilon-g-definition},
\eqref{eq:main-regularity}, and \eqref{eq:main-signal-radius}, and adjusting
the constants in the theorem,
\begin{equation}
\label{eq:main-proof-small-error-bounds}
\begin{gathered}
    \varepsilon_{\mathrm{fit}}\le C\varepsilon_{\mathrm{flat}}\le\frac1{16},
    \qquad
    C_{\mathrm{conc}}R_\xi^2\sigma^2\Delta_0\le \frac14\sigma^2,\\
    C_{\mathrm{conc}}R_X^2(b_S^+)^2\Delta_S\le \frac12(b_S^-)^2,
    \qquad
    C_{\mathrm{conc}}R_XR_\xi\sigma b_S^+\Delta_S
    \le \frac18\min\left\{\sigma^2,(b_S^-)^2r_\delta\right\}.
\end{gathered}
\end{equation}
Next, we will prove the following inequality (after adjusting constants if necessary) which will be needed in the proof,
\begin{equation}
\label{eq:main-proof-radius-dominates}
\begin{aligned}
    \frac18(b_S^-)^2r_\delta^2
    \ge 
    128\sigma^2\varepsilon_{\mathrm{fit}}
        &+16\sigma\gamma_N
            \left(2\|(\bw_S^\star,\zero)\|+\beta_Sr_\delta\right)
            +16\gamma_N^2
            \left(2\|(\bw_S^\star,\zero)\|+\beta_Sr_\delta\right)^2.
\end{aligned}
\end{equation}
To see this, after using \(\varepsilon_{\mathrm{fit}}\le C\varepsilon_{\mathrm{flat}}\) and
\((x+y)^2\le2x^2+2y^2\), the only terms on the right-hand side not directly
present in \eqref{eq:main-signal-radius} are controlled, for any
\(\eta>0\), by
\[
    \sigma\gamma_N\beta_Sr_\delta
    \le
    \eta(b_S^-)^2r_\delta^2
    +\frac{\sigma^2}{4\eta}
        \left(\frac{\gamma_N\beta_S}{b_S^-}\right)^2,
    \qquad
    \gamma_N^2\beta_S^2r_\delta^2
    \le
    c^2(b_S^-)^2r_\delta^2.
\]
Choose \(\eta\) and the constant \(c\) in
\eqref{eq:main-regularity} sufficiently small that, after accounting for the
fixed numerical coefficients in \eqref{eq:main-proof-radius-dominates}, the
two multiples of \((b_S^-)^2r_\delta^2\) sum to at most
\((b_S^-)^2r_\delta^2/16\). The remaining terms are all present in
\eqref{eq:main-signal-radius}; increasing its fixed constant makes their sum
at most \((b_S^-)^2r_\delta^2/16\), proving
\eqref{eq:main-proof-radius-dominates}.

Apply Lemma~\ref{lem:g-bound} and
Lemma~\ref{lem:basic-concentration-consequences} with \(T=S\), each with
failure probability \(\delta_0\). By a union bound, there is an event
\(\Ecal\) of probability at least \(1-\delta\) on which \(X_N\) has full row
rank and
\begin{equation}
\label{eq:main-proof-g-bound}
    (1-\varepsilon_{\mathrm{fit}})\|\bz(\bw_S)\|_2
    \le
    \kappa_Ng(\bw_S)
    \le
    (1+\varepsilon_{\mathrm{fit}})\|\bz(\bw_S)\|_2
    \qquad
    \text{for every } \bw_S\in\Wcal_S,
\end{equation}
and
\begin{equation}
\label{eq:main-proof-signal-event}
\begin{aligned}
    \left\|
        \frac1nX_S^\top X_S-\Sigma_S
    \right\|_{\op}
    &\le
    C_{\mathrm{conc}}R_X^2(b_S^+)^2\Delta_S,
    \\
    \left\|
        \frac1nX_S^\top\bxi
    \right\|_2
    &\le
    C_{\mathrm{conc}}R_XR_\xi\sigma b_S^+\Delta_S,
    \\
    \left|
        \frac1n\|\bxi\|_2^2-\sigma^2
    \right|
    &\le
    C_{\mathrm{conc}}R_\xi^2\sigma^2\Delta_0 .
\end{aligned}
\end{equation}
Work on \(\Ecal\).
Define
\[
    F(\bw_S)
    :=
    \min_{\bw_N:\,
        X_N\bw_N=\bz(\bw_S)+X_N\bw_N^\star}
    \|(\bw_S,\bw_N)\|,
\]
where the defining constraint is feasible for every \(\bw_S\) because
\(X_N\) has full row rank. The \(S\)-blocks of the minimum-norm interpolators
are precisely the minimizers of \(F\), with the corresponding \(N\)-blocks
attaining the inner minimum in the definition of \(F\).
For every
\(\bw_S\in\Wcal_S\), the triangle inequality gives
\begin{equation}
\label{eq:objective-sandwich}
    g(\bw_S)-\|(\bw_S,\zero)\|
    \le
    F(\bw_S)
    \le
    g(\bw_S)+\|(\bw_S,\zero)\|.
\end{equation}
By \eqref{eq:main-proof-g-bound}, \eqref{eq:objective-sandwich}, and
\(\gamma_N=\kappa_N/\sqrt n\),
\begin{equation}
\label{eq:F-lower}
    \gamma_NF(\bw_S)
    \ge
    (1-\varepsilon_{\mathrm{fit}})\frac{\|\bz(\bw_S)\|_2}{\sqrt n}
    -
    \gamma_N\|(\bw_S^\star,\zero)\|
    -
    \gamma_N\beta_S\|\bw_S-\bw_S^\star\|_2 .
\end{equation}
Evaluating the same bounds at \(\bw_S=\bw_S^\star\), where
\(\bz(\bw_S^\star)=\bxi\), gives
\begin{equation}
\label{eq:F-upper-star}
    \gamma_NF(\bw_S^\star)
    \le
    (1+\varepsilon_{\mathrm{fit}})\frac{\|\bxi\|_2}{\sqrt n}
    +
    \gamma_N\|(\bw_S^\star,\zero)\|.
\end{equation}
Subtracting the upper bound \eqref{eq:F-upper-star} from the lower bound
\eqref{eq:F-lower} yields, for every \(\bw_S\in\Wcal_S\),
\begin{equation}
\label{eq:F_diff}
\begin{aligned}
    \gamma_N\left(F(\bw_S)-F(\bw_S^\star)\right)
    \ge{}&
    (1-\varepsilon_{\mathrm{fit}})\frac{\|\bz(\bw_S)\|_2}{\sqrt n}
    -(1+\varepsilon_{\mathrm{fit}})\frac{\|\bxi\|_2}{\sqrt n}\\
    &\quad -
    \gamma_N\left(
        2\|(\bw_S^\star,\zero)\|
        +\beta_S\|\bw_S-\bw_S^\star\|_2
    \right).
\end{aligned}
\end{equation}

Now fix \(\bw_S\) satisfying
\(\|\bw_S-\bw_S^\star\|_2=r_\delta\), and write
\(\bu:=\bw_S^\star-\bw_S\). Since \(r_\delta\le1\), this point belongs to
\(\Wcal_S\), so \eqref{eq:F_diff} applies. On \(\Ecal\),
\begin{align*}
    \frac1n\|\bz(\bw_S)\|_2^2
    &=
    \bu^\top\left(\frac1nX_S^\top X_S\right)\bu
    +\frac2n\bu^\top X_S^\top\bxi
    +\frac1n\|\bxi\|_2^2\\
    &\ge
    \frac1n\|\bxi\|_2^2
    +\left(
        (b_S^-)^2-C_{\mathrm{conc}}R_X^2(b_S^+)^2\Delta_S
    \right)r_\delta^2
    -2C_{\mathrm{conc}}R_XR_\xi\sigma b_S^+\Delta_Sr_\delta.
\end{align*}
The bounds in \eqref{eq:main-proof-small-error-bounds} give
\[
    (b_S^-)^2-C_{\mathrm{conc}}R_X^2(b_S^+)^2\Delta_S
    \ge\frac12(b_S^-)^2,
    \qquad
    2C_{\mathrm{conc}}R_XR_\xi\sigma b_S^+\Delta_Sr_\delta
    \le\frac14(b_S^-)^2r_\delta^2.
\]
Therefore,
\begin{equation}
\label{eq:z-lower-radius-eps}
    \frac1{\sqrt n}\|\bz(\bw_S)\|_2
    \ge
    \sqrt{
        \frac1n\|\bxi\|_2^2
        +\frac14(b_S^-)^2r_\delta^2
    }.
\end{equation}
Since \(1-\varepsilon_{\mathrm{fit}}>0\), substituting
\eqref{eq:z-lower-radius-eps} into \eqref{eq:F_diff} yields
\begin{align}
\label{eq:F_diff2}
    \gamma_N\left(F(\bw_S)-F(\bw_S^\star)\right)
    &\ge
    (1-\varepsilon_{\mathrm{fit}})\sqrt{
        \frac1n\|\bxi\|_2^2
        +\frac14(b_S^-)^2r_\delta^2
    } 
    -(1+\varepsilon_{\mathrm{fit}})\frac{\|\bxi\|_2}{\sqrt n} \nonumber\\
    &\quad -
    \gamma_N\left(
        2\|(\bw_S^\star,\zero)\|
        +\beta_Sr_\delta
    \right) \nonumber\\ 
    &\overset{(*)}{\ge}
    \frac{(b_S^-)^2r_\delta^2}{
        8\sqrt{
            \frac1n\|\bxi\|_2^2
            +\frac14(b_S^-)^2r_\delta^2
        }
    }
    -2\varepsilon_{\mathrm{fit}}
        \sqrt{
            \frac1n\|\bxi\|_2^2
            +\frac14(b_S^-)^2r_\delta^2
        } \nonumber\\
    &\quad -
    \gamma_N\left(
        2\|(\bw_S^\star,\zero)\|+\beta_Sr_\delta
    \right).
\end{align}
Here \((*)\) uses, for \(a,b\ge0\) that
$\sqrt{a+b}-\sqrt a \ge \frac{b}{2\sqrt{a+b}}$,
and $\sqrt{a+b}+\sqrt a\le2\sqrt{a+b}.$
Combining \eqref{eq:main-proof-signal-event} and
\eqref{eq:main-proof-small-error-bounds} yields
\[
    \frac{\|\bxi\|_2^2}{n}
    \le
    \sigma^2+C_{\mathrm{conc}}R_\xi^2\sigma^2\Delta_0
    \le\frac54\sigma^2,
\]
and it follows that
\[
    \sqrt{
        \frac1n\|\bxi\|_2^2
        +\frac14(b_S^-)^2r_\delta^2
    }
    \le
    2\sigma+\frac12b_S^-r_\delta.
\]
Using this upper bound in both square-root terms in
\eqref{eq:F_diff2} gives
\[
\begin{aligned}
    \gamma_N\left(F(\bw_S)-F(\bw_S^\star)\right)
    \ge{}&
    \frac{(b_S^-)^2r_\delta^2}{
        8\left(2\sigma+\frac12b_S^-r_\delta\right)
    }
    -2\varepsilon_{\mathrm{fit}}\left(2\sigma+\frac12b_S^-r_\delta\right)
    -
    \gamma_N\left(
        2\|(\bw_S^\star,\zero)\|+\beta_Sr_\delta
    \right).
\end{aligned}
\]
Therefore, to ensure that \(F(\bw_S)>F(\bw_S^\star)\), it suffices to show
that
\begin{equation}
\label{eq:main-proof-radius-dominates-2}
    (b_S^-)^2r_\delta^2
    >
    16\varepsilon_{\mathrm{fit}}\left(2\sigma+\frac12b_S^-r_\delta\right)^2
    +8\gamma_N\left(
        2\|(\bw_S^\star,\zero)\|+\beta_Sr_\delta
    \right)
        \left(2\sigma+\frac12b_S^-r_\delta\right).
\end{equation}
Since \((x+y)^2\le2x^2+2y^2\) and \(\varepsilon_{\mathrm{fit}}\le1/16\), the first term on the right-hand side is at most $128\sigma^2\varepsilon_{\mathrm{fit}} +\frac12(b_S^-)^2r_\delta^2$. 
Next, let \(x:=b_S^-r_\delta\) and
\(y:=\gamma_N(2\|(\bw_S^\star,\zero)\|+\beta_Sr_\delta)\) so that the second term on the right-hand side is $16\sigma y+4xy$. Young's inequality gives
\(4xy\le x^2/4+16y^2\). Hence, the second term is at most
\[
    16\sigma\gamma_N\left(
        2\|(\bw_S^\star,\zero)\|+\beta_Sr_\delta
    \right)
    +\frac14(b_S^-)^2r_\delta^2
    +16\gamma_N^2\left(
        2\|(\bw_S^\star,\zero)\|+\beta_Sr_\delta
    \right)^2.
\]
Consequently, to prove \eqref{eq:main-proof-radius-dominates-2}, it suffices to
show that
\begin{align*}
    \frac14(b_S^-)^2r_\delta^2
    >
    &128\sigma^2\varepsilon_{\mathrm{fit}}
    +16\sigma\gamma_N\left(
        2\|(\bw_S^\star,\zero)\|+\beta_Sr_\delta
    \right)+16\gamma_N^2\left(
        2\|(\bw_S^\star,\zero)\|+\beta_Sr_\delta
    \right)^2.
\end{align*}
By \eqref{eq:main-proof-radius-dominates}, the right-hand side is at most
\((b_S^-)^2r_\delta^2/8\). Thus
\eqref{eq:main-proof-radius-dominates-2} holds, proving
\begin{equation}
\label{eq:F-boundary-separation}
    F(\bw_S)>F(\bw_S^\star)
    \qquad
    \text{whenever }
    \|\bw_S-\bw_S^\star\|_2=r_\delta .
\end{equation}
The function \(F\) is convex, hence
\eqref{eq:F-boundary-separation} implies that
every minimizer satisfies
\[
    \|\widehat{\bw}_S-\bw_S^\star\|_2<r_\delta .
\]
Indeed, otherwise the line segment between \(\bw_S^\star\) and a minimizer
outside the radius-\(r_\delta\) ball would intersect the sphere
\(\|\bw_S-\bw_S^\star\|_2=r_\delta\) at a point with objective value at most
\(F(\bw_S^\star)\), a contradiction.

It remains to bound the risk. Since \(F(\widehat{\bw}_S)\le F(\bw_S^\star)\),
the triangle inequality and the signal bound just proved give
\[
    \|(0,\widehat{\bw}_N)\|
    \le
    F(\widehat{\bw}_S)+\|(\widehat{\bw}_S,\zero)\|
    \le
    F(\bw_S^\star)+\|(\bw_S^\star,\zero)\|+\beta_Sr_\delta.
\]
Applying \eqref{eq:F-upper-star} and the \(\|\bxi\|_2^2\) bound in
\eqref{eq:main-proof-signal-event} therefore gives
\begin{equation}
\label{eq:what-N-norm-bound}
    \|(0,\widehat{\bw}_N)\|
    \le
    \frac{1+\varepsilon_{\mathrm{fit}}}{\gamma_N}
        \sqrt{\sigma^2+C_{\mathrm{conc}}R_\xi^2\sigma^2\Delta_0}
    +
    2\|(\bw_S^\star,\zero)\|
    +
    \beta_Sr_\delta .
\end{equation}
Finally, Assumption~\ref{assump:subgaussian-input-distribution} gives
\[
    \Rcal(\widehat{\bw})
    =
    \sum_{j\in S}b_j^2(\widehat w_j-w_j^\star)^2
    +
    \sum_{j\in N}b_j^2(\widehat w_j-w_j^\star)^2 .
\]
The first term on the right-hand side is at most \((b_S^+)^2r_\delta^2\). For the second term,
\[
\begin{aligned}
    \sum_{j\in N}b_j^2(\widehat w_j-w_j^\star)^2
    &\le
    2\sum_{j\in N}b_j^2\widehat w_j^2
    +
    2\sum_{j\in N}b_j^2(w_j^\star)^2 \le
    2\beta_{N,\Sigma}^2
    \|(0,\widehat\bw_N)\|^2
    +
    2\nu_N^2.
\end{aligned}
\]
Equations~\eqref{eq:main-proof-small-error-bounds}
and~\eqref{eq:what-N-norm-bound} imply
\[
    \|(0,\widehat{\bw}_N)\|
    \le
    C
    \left(
        \frac{\sigma}{\gamma_N}
        +
        \|(\bw_S^\star,\zero)\|
        +
        \beta_Sr_\delta
    \right).
\]
Substituting this bound proves the desired risk bound.
\end{proof}

\subsection{Proof of Corollary~\ref{cor:general-asymptotic}}

\begin{proof}
Since \(\varepsilon_{\mathrm{flat}}\ge\sqrt\rho\) and
\(\rho\ge\nu_N/\sigma\), the first limit in the corollary statement implies
\(\rho\to0\) and \(\nu_N\to0\). The definition of \(\rho\) and the uniform
bounds in the statement also imply
\[
    \frac{r_S+\log(n/\delta_n)}{n}\longrightarrow0.
\]
Moreover,
\(\|(\bw_S^\star,\zero)\|\le\beta_S\|\bw_S^\star\|_2\).
Thus, the assumptions of the corollary and the definitions in
\eqref{eq:main-scales}--\eqref{eq:main-signal-radius} imply that the regularity
condition in \eqref{eq:main-regularity} holds for all sufficiently large \(n\)
and that \(r_{\delta_n}\to0\).

It remains to apply the risk bound in Theorem~\ref{thm:main}. Its first and
last terms tend to zero, while
\[
\begin{aligned}
    &\beta_{N,\Sigma}
    \left(
        \frac{\sigma}{\gamma_N}
        +\|(\bw_S^\star,\zero)\|
        +\beta_Sr_{\delta_n}
    \right)
    =
    \frac{\beta_{N,\Sigma}}{\gamma_N}
    \left(
        \sigma
        +\gamma_N\|(\bw_S^\star,\zero)\|
        +\gamma_N\beta_Sr_{\delta_n}
    \right)
    \longrightarrow0.
\end{aligned}
\]
Applying Theorem~\ref{thm:main} with \(\delta=\delta_n\) proves the claim.
\end{proof}

\section{Proof of Theorem~\ref{thm:lp-clean} and Related Results} \label{app:lp_satisfies_assumps}
\subsection{Setting}
\label{setting:lp-product}
Fix \(p\in(1,\infty]\), and let \(q\in[1,\infty)\) be its dual exponent,
defined by \(1/p+1/q=1\). 

\begin{definition}\label{def:sub-weibull}
For \(\alpha>0\), define the \(\psi_\alpha\) Orlicz functional of a real-valued
random variable \(Y\) by
\[
    \|Y\|_{\psi_\alpha}
    :=
    \inf\left\{s>0:
    \E\exp\left(\left|\frac{Y}{s}\right|^\alpha\right)\le 2
    \right\}.
\]
When \(\alpha\ge1\) this is an Orlicz norm; for \(\alpha<1\) we use it as the
standard \(\psi_\alpha\) quasi-norm.
\end{definition}

We locally index the columns of \(X_N\) by
\([|N|]\). Throughout this section, \(s_N>0\), and the
noise-fitting coordinates have the product form
\[
    x_{i,k}=s_N\zeta_{i,k},
    \qquad i\in[n],\ k\in[|N|],
\]
where the variables \(\{\zeta_{i,k}\}_{i,k}\) are mutually independent and,
for every \(i\in[n]\) and \(k\in[|N|]\),
\[
    \E\zeta_{i,k}=0,
    \qquad
    \E\zeta_{i,k}^2=1,
    \qquad
    \|\zeta_{i,k}\|_{\psi_2}\le R_X .
\]
It is well known that any sub-Gaussian random variable satisfies $\norm{\zeta_{i, k}}_{\psi_2} \geq \sqrt{\E \zeta_{i, k}^2} = 1$, so $R_X \geq 1$.
For \(\blambda\in\R^n\) and \(k\in[|N|]\), write
\[
    b_{\blambda,k}
    :=
    \sum_{i=1}^n \lambda_i\zeta_{i,k}.
\]
For fixed \(\blambda\), the variables
\(\{b_{\blambda,k}\}_{k=1}^{|N|}\) are independent. Moreover,
\[
    X_N^\top\blambda
    =
    s_N(b_{\blambda,1},\ldots,b_{\blambda,|N|}).
\]

\begin{lemma}[Centering a \(\psi_\alpha\) random variable]
\label{lem:psi-alpha-centering}
There is an absolute constant \(C_{\mathrm{ctr}}>0\) such that, for every
\(\alpha\in(0,2]\) and every real-valued random variable \(Y\) with
\(\|Y\|_{\psi_\alpha}<\infty\),
\begin{equation}
\label{eq:psi-alpha-centering}
    \|Y-\E Y\|_{\psi_\alpha}
    \le
    \left(\frac{C_{\mathrm{ctr}}}{\alpha}\right)^{1/\alpha}
    \|Y\|_{\psi_\alpha}.
\end{equation}
\end{lemma}

\begin{proof}
Fix \(K>\|Y\|_{\psi_\alpha}\). Markov's inequality and Def.~\ref{def:sub-weibull} give for any $t>0$,
\[
\begin{aligned}
    \Pr(|Y|\ge Kt)
    &=
    \Pr\left(
        \exp\left(\left|\frac{Y}{K}\right|^\alpha\right)
        \ge e^{t^\alpha}
    \right)
    \le
    e^{-t^\alpha}
    \E\exp\left(\left|\frac{Y}{K}\right|^\alpha\right)
    \le
    2e^{-t^\alpha},
\end{aligned}
\]
and hence, by tail integration,
\[
    |\E Y|
    \le
    \E|Y|
    \le
    2K\int_0^\infty e^{-t^\alpha}\,dt
    =
    2\Gamma(1+1/\alpha)K,
\]
where \(\Gamma\) is the Euler gamma function. By Stirling's inequality,
\(\Gamma(1+1/\alpha)\le(C/\alpha)^{1/\alpha}\) for
\(0<\alpha\le2\). As such, there is an
absolute constant \(C_1>0\) such that
\[
    |\E Y|^\alpha
    \le
    \frac{C_1}{\alpha}K^\alpha,
    \qquad \text{where}~~~ 0<\alpha\le2.
\]
Also, \(\alpha\le2\) implies
\(|x-y|^\alpha\le2(|x|^\alpha+|y|^\alpha)\). Therefore, for
\(D_\alpha:=(C_{\mathrm{ctr}}/\alpha)^{1/\alpha}\), where
\(C_{\mathrm{ctr}}\) will be chosen below,
\[
    \left|\frac{Y-\E Y}{D_\alpha K}\right|^\alpha
    \le
    \frac{2\alpha}{C_{\mathrm{ctr}}}
    \left|\frac{Y}{K}\right|^\alpha
    +
    \frac{2C_1}{C_{\mathrm{ctr}}}.
\]
Choose the absolute constant \(C_{\mathrm{ctr}}\) large enough that
\[
    \frac{2\alpha}{C_{\mathrm{ctr}}}\le\frac12
    \quad\text{for all }\alpha\le2,
    \qquad
    \frac{2C_1}{C_{\mathrm{ctr}}}\le\frac{\log 2}{2}.
\]
Exponentiating the preceding bound and applying Jensen's inequality now gives
\begin{align*}
    \E\exp\left(
        \left|\frac{Y-\E Y}{D_\alpha K}\right|^\alpha
    \right)
    &\le
    e^{(\log 2)/2}
    \left[
        \E\exp\left(\left|\frac{Y}{K}\right|^\alpha\right)
    \right]^{1/2}
    \le2.
\end{align*}
Letting \(K\downarrow\|Y\|_{\psi_\alpha}\) proves
\eqref{eq:psi-alpha-centering}.
\end{proof}

\subsection{Moments and Expected Dual-Norm Geometry}

\begin{lemma}[Moments of sub-Gaussian linear forms]
\label{lem:subgaussian-linear-moments}
Under Setting~\ref{setting:lp-product}, there exist absolute constants
\(c,C>0\) such that, for every \(r\ge1\), every \(k\in[|N|]\), and every
\(\blambda\in\Sphere^{n-1}\),
\[
    \frac1{\sqrt2}
    \left(\frac{c}{R_X^4}\right)^{1/r}
    \le
    \left(\E |b_{\blambda,k}|^r\right)^{1/r}
    \le
    C R_X\sqrt r.
\]
\end{lemma}

\begin{proof}
Choose an absolute constant \(C_0\ge1\) large enough that the standard
sub-Gaussian sum and moment bounds
\citep[Propositions 2.7.1 and 2.6.1]{Vershynin_2026} give
\(\|b_{\blambda,k}\|_{\psi_2}\le C_0R_X\) and
\[
    \left(\E |b_{\blambda,k}|^r\right)^{1/r}
    \le
    C_0 R_X\sqrt r.
\]
For the lower bound, \(\E b_{\blambda,k}^2=1\), and the upper bound with
\(r=4\) gives \(\E b_{\blambda,k}^4\le (2C_0R_X)^4\). Applying
Paley--Zygmund to \(b_{\blambda,k}^2\) yields
\[
    \Pr\left(|b_{\blambda,k}|\ge \frac1{\sqrt2}\right)
    \ge
    \frac{1}{4(2C_0R_X)^4}.
\]
Set \(c_0:=1/[4(2C_0)^4]\). The elementary inequality
\(\E Z\ge t\Pr(Z\ge t)\), applied to
\(Z=|b_{\blambda,k}|^r\), gives
\[
    \E |b_{\blambda,k}|^r
    \ge
    2^{-r/2}
    \frac{c_0}{R_X^4}.
\]
Taking the \(r\)-th root proves the result with \(c=c_0\) and \(C=C_0\).
\end{proof}

\begin{lemma}[Expected \(\ell_q\) norm]
\label{lem:exnl_subgaussian}
Under Setting~\ref{setting:lp-product}, there exist absolute constants
\(c,C>0\) such that, for every \(\blambda\in\Sphere^{n-1}\),
\[
    \frac{c}{R_X^4}s_N |N|^{1/q}
    \le
    \E\|X_N^\top\blambda\|_q
    \le
    C R_X\sqrt q\,s_N |N|^{1/q}.
\]
\end{lemma}

\begin{proof}
Take \(C\) to be the upper constant in
Lemma~\ref{lem:subgaussian-linear-moments}, and take \(c\) to be
\(1/\sqrt2\) times that lemma's lower constant.
By Jensen's inequality, using the concavity of \(u\mapsto u^{1/q}\),
\[
    \E\|X_N^\top\blambda\|_q
    =
    s_N
    \E\left(\sum_{k=1}^{|N|}|b_{\blambda,k}|^q\right)^{1/q}
    \le
    s_N
    \left(\sum_{k=1}^{|N|}\E |b_{\blambda,k}|^q\right)^{1/q}
    \le
    C R_X\sqrt q\,s_N |N|^{1/q}.
\]
For the lower bound, use
\(\|\bv\|_q\ge |N|^{1/q-1}\|\bv\|_1\) for
\(\bv\in\R^{N}\). Then
\[
    \E\|X_N^\top\blambda\|_q
    \ge
    s_N |N|^{1/q-1}
    \sum_{k=1}^{|N|}\E |b_{\blambda,k}|
    \ge
    \frac{c}{R_X^4}s_N |N|^{1/q}.
\]
\end{proof}

\subsection{Concentration of the Empirical Dual Norm}

\begin{lemma}[Sub-Weibull mean concentration]
\label{lem:subweibull-mean-conc}
Let \(m\ge1\), and let \(Y_1,\ldots,Y_m\) be independent mean-zero random
variables with \(\|Y_i\|_{\psi_\alpha}\le B\) for some \(\alpha,B>0\). Then,
for every \(u>0\),
\[
    \Pr\left(
        \left|\frac1m\sum_{i=1}^m Y_i\right|
        \ge u
    \right)
    \le
    2\exp\left(
        -c_\alpha
        \min\left\{
            \frac{m u^2}{B^2},
            \left(
                \frac{m^{\min(1/\alpha,1)}u}{B}
            \right)^\alpha
        \right\}
    \right),
\]
where \(c_\alpha>0\) depends only on \(\alpha\) and is non-decreasing for
\(\alpha\in(0,2]\).
\end{lemma}

\begin{proof}
By \cite[Theorem 3.1]{kuchibhotla2022moving}, there is a constant
\(C_\alpha>0\), depending only on \(\alpha\), such that, for every
\(t\ge0\),
\[
    \Pr\left(
        \left|\frac1m\sum_{i=1}^m Y_i\right|
        \ge
        C_\alpha B
        \left(
            \sqrt{\frac{t}{m}}
            +
            \frac{t^{1/\alpha}}{m^{\min(1/\alpha,1)}}
        \right)
    \right)
    \le
    2e^{-t}.
\]
Apply this bound with
\[
    t
    =
    c_\alpha
    \min\left\{
        \frac{m u^2}{B^2},
        \left(
            \frac{m^{\min(1/\alpha,1)}u}{B}
        \right)^\alpha
    \right\},
\]
where \(c_\alpha>0\) is chosen small enough that
\(C_\alpha(c_\alpha^{1/2}+c_\alpha^{1/\alpha})\le1\). Since
\[
    B\sqrt{\frac{t}{m}}
    \le
    \sqrt{c_\alpha}u,
    \qquad
    B\frac{t^{1/\alpha}}{m^{\min(1/\alpha,1)}}
    \le
    c_\alpha^{1/\alpha}u,
\]
the threshold in the cited bound is at most \(u\), proving the stated tail
inequality.

Finally, the explicit formula for \(C_\alpha\) in the cited theorem is bounded
above on every interval \([\alpha_0,2]\) by a constant depending only on
\(\alpha_0\). Hence the admissible choices of \(c_\alpha\) can be made
uniformly positive on every such interval. For \(\alpha\in(0,2]\), replace
each choice by the infimum of the chosen constants over \([\alpha,2]\). The
resulting constants remain positive and admissible, and are non-decreasing in
\(\alpha\), as claimed.
\end{proof}

\begin{lemma}[Bias of the empirical \(q\)-root]
\label{lem:empirical-root-bias}
Let \(Y_1,\ldots,Y_m\) be independent nonnegative random variables with
\(\E Y_i^2<\infty\). Then, for every \(q\ge1\),
\[
    \left|
        \E\left(\frac1m\sum_{i=1}^m Y_i\right)^{1/q}
        -
        \left(\frac1m\sum_{i=1}^m\E Y_i\right)^{1/q}
    \right|
    \le
    \left(\frac1{m^2}\sum_{i=1}^m\Var(Y_i)\right)^{1/(2q)}.
\]
\end{lemma}

\begin{proof}
For \(u,v\ge0\), the map \(u\mapsto u^{1/q}\) is \(1/q\)-Hölder, so
\[
    |u^{1/q}-v^{1/q}|\le |u-v|^{1/q}.
\]
Using this and Jensen's inequality,
\[
\begin{aligned}
    \left|
        \E\left(\frac1m\sum_{i=1}^m Y_i\right)^{1/q}
        -
        \left(\frac1m\sum_{i=1}^m\E Y_i\right)^{1/q}
    \right|
    &\le
    \E\left|\frac1m\sum_{i=1}^m(Y_i-\E Y_i)\right|^{1/q} \\
    &\le
    \left[
        \E\left|\frac1m\sum_{i=1}^m(Y_i-\E Y_i)\right|^2
    \right]^{1/(2q)} \\
    &=
    \left(\frac1{m^2}\sum_{i=1}^m\Var(Y_i)\right)^{1/(2q)}.
\end{aligned}
\]
The equality uses the independence and centering of the variables
\(Y_i-\E Y_i\).
\end{proof}

\begin{lemma}[Fixed-vector concentration]
\label{lem:fixed-vector-lq-concentration}
Under Setting~\ref{setting:lp-product}, there is an absolute constant \(C>0\)
and, for every \(q\ge1\), a constant \(c(q)>0\) that may depend on \(q\) such that for every
\(\blambda\in\Sphere^{n-1}\) and every \(\tau\in(0,1)\),
\[
\begin{aligned}
    &\Pr\left(
        \left|
            \frac{\|X_N^\top\blambda\|_q}
            {\E\|X_N^\top\blambda\|_q}
            -1
        \right|
        >
        C R_X^5\sqrt q
        \left(\tau+|N|^{-1/(2q)}\right)
    \right)\\
    &\quad\le
    2\exp\left(
        -\frac{c(q)}{R_X^{2q+8}}
        \min\left\{
            |N|\tau^2,
            (|N|\tau)^{2/q}
        \right\}
    \right).
\end{aligned}
\]
Moreover, for every \(q_0<\infty\), the constant \(c(q)\) is bounded below by
a positive constant depending only on \(q_0\), uniformly over
\(q\in[1,q_0]\).
\end{lemma}

\begin{proof}
Fix
\(\blambda\in\Sphere^{n-1}\), and write
\[
    Y_k:=|b_{\blambda,k}|^q,
    \qquad
    \bar\mu:=\frac1{|N|}\sum_{k=1}^{|N|}\E Y_k.
\]
By Lemma~\ref{lem:subgaussian-linear-moments}, there are absolute constants
\(c,C>0\) such that
\begin{equation}
\label{eq:fixed-mu-bounds}
    2^{-q/2}\frac{c}{R_X^4}
    \le \E Y_k
    \le (CR_X\sqrt q)^q
    \quad \text{for every }k\in[|N|],
    \qquad
    2^{-q/2}\frac{c}{R_X^4}
    \le \bar\mu
    \le (CR_X\sqrt q)^q .
\end{equation}
The standard sub-Gaussian sum bound gives
\(\|b_{\blambda,k}\|_{\psi_2}\le CR_X\), and hence
\[
    \|Y_k\|_{\psi_{2/q}}
    =
    \|b_{\blambda,k}\|_{\psi_2}^q
    \le
    (CR_X)^q.
\]
Applying Lemma~\ref{lem:psi-alpha-centering} with \(\alpha=2/q\) yields
\[
    \|Y_k-\E Y_k\|_{\psi_{2/q}}
    \le (CR_X\sqrt q)^q,
\]
after adjusting the absolute constant \(C\).

Set \(B:=(CR_X\sqrt q)^q\). The bounds in
\eqref{eq:fixed-mu-bounds} show that there is a constant \(a(q)>0\), depending
only on \(q\), such that
\[
    \frac{\bar\mu}{B}
    \ge
    \frac{a(q)}{R_X^{q+4}}.
\]
We may assume that \(a(q)\le1\). Moreover, for every \(q_0<\infty\), it can
be chosen uniformly bounded below over \(q\in[1,q_0]\).
Applying
Lemma~\ref{lem:subweibull-mean-conc} with \(\alpha=2/q\) and
\(u=\tau\bar\mu\) gives
\[
\Pr\left(
    \left|\frac1{|N|}\sum_{k=1}^{|N|}Y_k-\bar\mu\right|
    >\tau\bar\mu
\right)
\le
2\exp\left(
    -c_{2/q}
    \min\left\{
        \frac{|N|\tau^2\bar\mu^2}{B^2},
        \left(
            \frac{|N|^{\min(q/2,1)}\tau\bar\mu}{B}
        \right)^{2/q}
    \right\}
\right).
\]
If \(q\ge2\), the second term in the minimum is
\((|N|\tau\bar\mu/B)^{2/q}\). If \(1\le q\le2\), it is
\(|N|\tau^{2/q}(\bar\mu/B)^{2/q}\), which is at least
\(|N|\tau^2(\bar\mu/B)^{2/q}\) because \(\tau\in(0,1)\). In this case,
\((|N|\tau)^{2/q}\ge |N|\tau^2\) as well. Since
\(R_X\ge1\) and \(a(q)\le1\), the preceding lower bound for
\(\bar\mu/B\) therefore gives, in both cases,
\[
\begin{split}
&\min\left\{
    \frac{|N|\tau^2\bar\mu^2}{B^2},
    \left(
        \frac{|N|^{\min(q/2,1)}\tau\bar\mu}{B}
    \right)^{2/q}
\right\} 
\ge
\frac{a(q)^2}{R_X^{2q+8}}
\min\left\{|N|\tau^2,(|N|\tau)^{2/q}\right\}.
\end{split}
\]
Set \(c(q):=c_{2/q}a(q)^2\).
The preceding bound becomes
\begin{equation}
\label{eq:fixed-sample-tail}
    \Pr\left(
        \left|
            \frac1{|N|}\sum_{k=1}^{|N|}Y_k
            -
            \bar\mu
        \right|
        >
        \tau\bar\mu
    \right)
    \le
    2\exp\left(
        -\frac{c(q)}{R_X^{2q+8}}
        \min\left\{
            |N|\tau^2,
            (|N|\tau)^{2/q}
        \right\}
    \right).
\end{equation}

Suppose that
\[
    \left|
        \frac1{|N|}\sum_{k=1}^{|N|}Y_k-\bar\mu
    \right|
    \le\tau\bar\mu.
\]
Using \(|s^{1/q}-1|\le|s-1|\),
Lemma~\ref{lem:empirical-root-bias}, and the triangle inequality then gives
\[
\begin{aligned}
    \left|
        \|X_N^\top\blambda\|_q
        -\E\|X_N^\top\blambda\|_q
    \right|
    &\le
    s_N|N|^{1/q}\bar\mu^{1/q}
    \left|
        \left(
            \frac{|N|^{-1}\sum_{k=1}^{|N|}Y_k}{\bar\mu}
        \right)^{1/q}-1
    \right| \\
    &\qquad
    +s_N|N|^{1/q}
        \left(
            \frac1{|N|^2}\sum_{k=1}^{|N|}\Var(Y_k)
        \right)^{1/(2q)} \\
    &\le
    \tau s_N|N|^{1/q}\bar\mu^{1/q}
    +s_N|N|^{1/q}
        \left(
            \frac1{|N|^2}\sum_{k=1}^{|N|}\Var(Y_k)
        \right)^{1/(2q)} \\
    &\le
    CR_X\sqrt q\,s_N|N|^{1/q}
    \left(\tau+|N|^{-1/(2q)}\right),
\end{aligned}
\]
where the last inequality uses \eqref{eq:fixed-mu-bounds} and the upper moment
bound in Lemma~\ref{lem:subgaussian-linear-moments}, applied with \(2q\),
after adjusting the absolute constant \(C\).
Moreover, Lemma~\ref{lem:exnl_subgaussian} gives
\(\E\|X_N^\top\blambda\|_q
\ge(c/R_X^4)s_N|N|^{1/q}\). Hence, after another adjustment of \(C\),
\[
    \left|
        \frac{\|X_N^\top\blambda\|_q}
        {\E\|X_N^\top\blambda\|_q}-1
    \right|
    \le
    CR_X^5\sqrt q
    \left(\tau+|N|^{-1/(2q)}\right).
\]
Thus, the deviation in the statement can occur only if the sample-mean
deviation exceeds \(\tau\bar\mu\). Its probability is bounded by
\eqref{eq:fixed-sample-tail}, proving the stated deviation bound.

Finally, if \(q\le q_0\), the monotonicity of the constants
in Lemma~\ref{lem:subweibull-mean-conc} gives
\(c_{2/q}\ge c_{2/q_0}>0\). Moreover, \(a(q)\) is bounded below by a
positive constant depending only on \(q_0\) over the same interval. Hence so
is \(c(q)=c_{2/q}a(q)^2\), as claimed.
\end{proof}

\begin{proposition}[Uniform concentration]
\label{prop:xnl_conc_subgaussian}
Under Setting~\ref{setting:lp-product}, there is an absolute constant \(C>0\)
and, for every \(q\ge1\), a constant \(c(q)>0\) that may depend on \(q\) such that for every
\(\tau,\delta\in(0,1)\), with
\(L=\sqrt n+\sqrt{\log(2|N|/\delta)}\),
\[
\begin{aligned}
    &\Pr\left(
        \sup_{\blambda\in\Sphere^{n-1}}
        \left|
            \frac{\|X_N^\top\blambda\|_q}
            {\E\|X_N^\top\blambda\|_q}
            -1
        \right|
        >
        CR_X^5\sqrt q\left(\tau+|N|^{-1/(2q)}\right)
    \right)
    \\
    &\qquad\le
    \delta
    +
    2\exp\left(
        n\log\left(\frac{CR_X^{10}\sqrt q\,L}{\tau}\right)
        -
        \frac{c(q)}{R_X^{2q+8}}
        \min\left\{
            |N|\tau^2,
            (|N|\tau)^{2/q}
        \right\}
    \right).
\end{aligned}
\]
Moreover, for every \(q_0<\infty\), \(c(q)\) is bounded below by a positive
constant depending only on \(q_0\), uniformly over \(q\in[1,q_0]\).
\end{proposition}

\begin{proof}
Choose absolute constants \(0<c_0\le1\le C_0\) so that the expectation bounds
in Lemma~\ref{lem:exnl_subgaussian} and the sub-Gaussian tail bound below
hold. Let
\(\bz_N\in\R^{n\times N}\) be the matrix with entries \(\zeta_{i,k}\), so
that \(X_N=s_N\bz_N\). For each column, the sub-Gaussian linear-form bound
\citep[Proposition 2.7.1]{Vershynin_2026} and a Euclidean \(1/2\)-net of
\(\Sphere^{n-1}\) give
\[
    \Pr\left(
        \|(\bz_N)_{:,k}\|_2
        > \frac{C_0}{2}R_X(\sqrt n+\sqrt t)
    \right)
    \le 2e^{-t},
    \qquad t\ge0.
\]
Taking \(t=\log(2|N|/\delta)\) and applying a union bound over
\(k\in[|N|]\) show that, with probability at least \(1-\delta\),
\begin{equation}
\label{eq:uniform-column-event}
    \max_{k\in[|N|]}\|(\bz_N)_{:,k}\|_2
    \le
    C_0R_XL .
\end{equation}
We work on this event.
For every \(\blambda,\blambda'\in\Sphere^{n-1}\),
\begin{equation}
\label{eq:uniform-numerator-lip}
    \|X_N^\top(\blambda-\blambda')\|_q
    \le
    C_0R_Xs_NL|N|^{1/q}
    \|\blambda-\blambda'\|_2 .
\end{equation}
Lemma~\ref{lem:exnl_subgaussian} and homogeneity also give
\begin{equation}
\label{eq:uniform-expectation-bounds}
    \E\|X_N^\top\blambda\|_q
    \ge
    \frac{c_0}{R_X^4}s_N|N|^{1/q},
    \qquad
    \E\|X_N^\top\bv\|_q
    \le
    C_0R_X\sqrt q\,s_N|N|^{1/q}\|\bv\|_2
\end{equation}
for every \(\blambda\in\Sphere^{n-1}\) and \(\bv\in\R^n\). Moreover,
\eqref{eq:uniform-column-event} implies
\begin{equation}
\label{eq:uniform-column-norm-bound}
    \|X_N^\top\blambda\|_q
    \le
    C_0R_Xs_NL|N|^{1/q}
    \qquad
    \text{for every }\blambda\in\Sphere^{n-1}.
\end{equation}
Define
\(\Phi(\blambda):=
\|X_N^\top\blambda\|_q/\E\|X_N^\top\blambda\|_q\),
and set \(C_1:=C_0/c_0+C_0^2/c_0^2\). For
\(\blambda,\blambda'\in\Sphere^{n-1}\), equations
\eqref{eq:uniform-numerator-lip}--\eqref{eq:uniform-column-norm-bound} give
\begin{align}
\label{eq:uniform-lipschitz}
    |\Phi(\blambda)-\Phi(\blambda')|
    &\le
    \frac{
        \|X_N^\top(\blambda-\blambda')\|_q
    }{
        \E\|X_N^\top\blambda\|_q
    } 
    +
    \|X_N^\top\blambda'\|_q
    \frac{
        \E\|X_N^\top(\blambda-\blambda')\|_q
    }{
        \E\|X_N^\top\blambda\|_q
        \E\|X_N^\top\blambda'\|_q
    } \notag\\
    &\le
    C_1R_X^{10}\sqrt q\,L\|\blambda-\blambda'\|_2 .
\end{align}

Let \(\mathcal N_r\) be a Euclidean \(r\)-net of \(\Sphere^{n-1}\), where
\[
    r
    :=
    \frac{\tau+|N|^{-1/(2q)}}
    {4C_1R_X^{10}\sqrt q\,L}.
\]
Since \(r<1\) and \(\tau+|N|^{-1/(2q)}\ge\tau\), the standard volumetric
bound allows us to choose the net so that
\begin{equation}
\label{eq:uniform-net-size}
    |\mathcal N_r|
    \le
    \left(\frac{12C_1R_X^{10}\sqrt q\,L}{\tau}\right)^n .
\end{equation}

Let \(C_2>0\) be the absolute constant in
Lemma~\ref{lem:fixed-vector-lq-concentration}, and choose the absolute constant
\(C\) in the proposition so that
\[
    C\ge\max\{12C_1,C_2+1/4\}.
\]
If
\[
    \sup_{\blambda\in\Sphere^{n-1}}|\Phi(\blambda)-1|
    >
    CR_X^5\sqrt q\left(\tau+|N|^{-1/(2q)}\right),
\]
then \eqref{eq:uniform-lipschitz} and the definition of \(r\) imply that
there is some \(\blambda_0\in\mathcal N_r\) with
\begin{equation}
\label{eq:uniform-net-reduction}
    |\Phi(\blambda_0)-1|
    >
    C_2R_X^5\sqrt q\left(\tau+|N|^{-1/(2q)}\right).
\end{equation}
The proposition now follows from a union bound, using
Lemma~\ref{lem:fixed-vector-lq-concentration},
\eqref{eq:uniform-net-size}, and the probability \(\delta\) that
\eqref{eq:uniform-column-event} fails. We take \(c(q)\) to be the constant in
Lemma~\ref{lem:fixed-vector-lq-concentration}, which also gives the asserted
uniform lower bound on \(c(q)\) over bounded intervals.
\end{proof}

\subsection{Gaussian Approximation via a Quantitative CLT}

For \(r\ge1\), let \(W_r\) denote the \(r\)-Wasserstein distance on probability
measures on \(\R\):
\[
    W_r(\mu,\nu)
    =
    \inf_{\pi\in\Pi(\mu,\nu)}
    \left(
        \int_{\R\times \R} \abs{x-y}^r\,d\pi(x,y)
    \right)^{1/r},
\]
where \(\Pi(\mu,\nu)\) is the set of all couplings of \(\mu\) and \(\nu\).

\begin{theorem}[\cite{bobkov2018berry}, Theorem 1.1]
\label{thm:bobkov}
Fix $r\geq 1$, and let \(U_1,\ldots,U_n\) be independent real-valued random variables such that
\(\E U_i=0\) and \(\E |U_i|^{r+2}<\infty\). Set
\(V=\sum_{i=1}^n U_i\), and assume \(\E V^2=1\). There exists a constant \(A(r)>0\), which may be chosen continuously in \(r\),
such that
\[
    W_r\left(\mathcal L(V),\mathcal N(0,1)\right)
    \le
    A(r)
    \left(
        \sum_{i=1}^n \E |U_i|^{r+2}
    \right)^{1/r}.
\]
\end{theorem}

We will use this theorem only through the following moment comparison, stated in
a slightly more general non-identically distributed form.

\begin{lemma}[Gaussian moment comparison for sub-Gaussian linear forms]
\label{lem:moment-comparison}
Let \(\zeta_1,\ldots,\zeta_n\) be independent random variables satisfying
\[
    \E \zeta_i=0,
    \qquad
    \E \zeta_i^2=1,
    \qquad
    \|\zeta_i\|_{\psi_2}\le R_X.
\]
There is a locally bounded function \(M:[1,\infty)\to(0,\infty)\) such that
for every \(q\ge1\), every
\(\blambda\in\Sphere^{n-1}\), and \(G\sim\mathcal N(0,1)\),
\[
    \left|
        \E\left|\sum_{i=1}^n\lambda_i\zeta_i\right|^q
        -\E |G|^q
    \right|
    \le
    M(q)R_X^{q+2/q}\|\blambda\|_\infty.
\]
\end{lemma}

\begin{proof}
Fix an absolute constant \(C_0\ge1\) large enough for the standard
sub-Gaussian moment and linear-form bounds used below, and set
\(W:=\sum_{i=1}^n\lambda_i\zeta_i\). Then \(\E W=0\) and \(\E W^2=1\).
Moreover,
\[
    \sum_{i=1}^n \E |\lambda_i\zeta_i|^{q+2}
    \le
    \left(C_0R_X\sqrt{q+2}\right)^{q+2}
    \sum_{i=1}^n|\lambda_i|^{q+2}
    \le
    \left(C_0R_X\sqrt{q+2}\right)^{q+2}
    \|\blambda\|_\infty^q .
\]
Theorem~\ref{thm:bobkov}, applied with \(U_i=\lambda_i\zeta_i\), therefore
gives
\begin{equation}
\label{eq:wasserstein-linear-form}
    W_q\left(\mathcal L(W),\mathcal N(0,1)\right)
    \le
    A(q)\left(C_0R_X\sqrt{q+2}\right)^{1+2/q}
    \|\blambda\|_\infty.
\end{equation}
For any coupling \((\widetilde W,\widetilde G)\) with
\(\widetilde W\sim\mathcal L(W)\) and
\(\widetilde G\sim\mathcal N(0,1)\), the reverse triangle inequality in
\(L_q\) gives
\[
    \left|
        \left(\E |W|^q\right)^{1/q}
        -
        \left(\E |G|^q\right)^{1/q}
    \right|
    \le
    \|\widetilde W-\widetilde G\|_{L_q}.
\]
Taking the infimum over all couplings and using
\eqref{eq:wasserstein-linear-form},
\[
    \left|
        \left(\E |W|^q\right)^{1/q}
        -
        \left(\E |G|^q\right)^{1/q}
    \right|
    \le
    A(q)\left(C_0R_X\sqrt{q+2}\right)^{1+2/q}
    \|\blambda\|_\infty.
\]
The linear form \(W\) is \(C_0R_X\)-sub-Gaussian. As noted after
Setting~\ref{setting:lp-product}, \(R_X\ge1\), so
\[
    \left(\E |W|^q\right)^{1/q}\le C_0R_X\sqrt q,
    \qquad
    \left(\E |G|^q\right)^{1/q}\le C_0R_X\sqrt q.
\]
Using \(|x^q-y^q|\le q(x+y)^{q-1}|x-y|\) for \(x,y\ge0\), we obtain
\[
    \left|
        \E |W|^q-\E |G|^q
    \right|
    \le
    M(q)R_X^{q+2/q}\|\blambda\|_\infty,
\]
where
\[
    M(q)
    :=
    q(2C_0\sqrt q)^{q-1}
    A(q)\left(C_0\sqrt{q+2}\right)^{1+2/q}.
\]
Since \(A(q)\) is continuous, \(M(q)\) is locally bounded.
\end{proof}

\subsubsection{Expected \texorpdfstring{\(\ell_q\)}{lq} geometry on flat directions}

\begin{lemma}[Expected \(\ell_q\) geometry on flat directions]
\label{lem:K_subgaussian}
Under Setting~\ref{setting:lp-product}, let \(G\sim\mathcal N(0,1)\),
\(m_q:=\E|G|^q\), and \(\kappa_N=s_N(|N|m_q)^{1/q}\). There are absolute
constants \(c,C>0\) and a locally bounded function
\(D:[1,\infty)\to(0,\infty)\) such that every nonzero
\(\blambda\in\R^n\) satisfies
\[
    \frac{\E\|X_N^\top\blambda\|_q}
    {\kappa_N\|\blambda\|_2}
    \ge
    \frac{c}{R_X^4\sqrt q}
\]
and
\[
    \left|
        \frac{\E\|X_N^\top\blambda\|_q}
        {\kappa_N\|\blambda\|_2}
        -1
    \right|
    \le
    CR_X|N|^{-1/(2q)}
    +
    D(q)R_X^{q+4-2/q}
    \frac{\|\blambda\|_\infty}{\|\blambda\|_2}.
\]
\end{lemma}

\begin{proof}
By homogeneity, it suffices to consider
\(\blambda\in\Sphere^{n-1}\). Lemma~\ref{lem:exnl_subgaussian} and the
standard bound \(m_q^{1/q}\le C\sqrt q\) give, for absolute constants
\(c,C>0\),
\[
    \frac{\E\|X_N^\top\blambda\|_q}{\kappa_N}
    \ge
    \frac{c}{R_X^4\sqrt q}.
\]
For the second bound, set
\[
    Z:=\frac1{|N|}\sum_{k=1}^{|N|}|b_{\blambda,k}|^q,
    \qquad
    \mu:=\E Z.
\]
Since
\(\E\|X_N^\top\blambda\|_q/\kappa_N=m_q^{-1/q}\E Z^{1/q}\),
the triangle inequality gives
\[
\begin{aligned}
    \left|
        \frac{\E\|X_N^\top\blambda\|_q}{\kappa_N}-1
    \right|
    &\le
    m_q^{-1/q}\left|\E Z^{1/q}-\mu^{1/q}\right| 
    +m_q^{-1/q}\left|\mu^{1/q}-m_q^{1/q}\right|.
\end{aligned}
\]
Lemma~\ref{lem:empirical-root-bias}, applied to the variables
\(|b_{\blambda,k}|^q\), together with
Lemma~\ref{lem:subgaussian-linear-moments} applied with \(r=2q\) and the
Gaussian lower bound \(m_q^{1/q}\ge c\sqrt q\), gives
\[
    m_q^{-1/q}
    \left|\E Z^{1/q}-\mu^{1/q}\right|
    \le
    CR_X|N|^{-1/(2q)}.
\]

To bound the second term, we apply Lemma~\ref{lem:moment-comparison}
separately to each \(b_{\blambda,k}\), obtaining
\[
    |\mu-m_q|
    \le
    M(q)R_X^{q+2/q}\|\blambda\|_\infty.
\]
Set \(a(q):=\min\{2^{-q/2}c,m_q\}\).
The lower moment bound in Lemma~\ref{lem:subgaussian-linear-moments} and
\(R_X\ge1\) imply that both \(\mu\) and \(m_q\) are at least
\(a(q)R_X^{-4}\). The mean value theorem therefore gives
\[
    m_q^{-1/q}
    \left|\mu^{1/q}-m_q^{1/q}\right|
    \le
    \frac{a(q)^{1/q-1}}{q\,m_q^{1/q}}
    R_X^{4-4/q}|\mu-m_q|
    \le
    D(q)R_X^{q+4-2/q}\|\blambda\|_\infty,
\]
where we may take
\(D(q)=a(q)^{1/q-1}M(q)/(q\,m_q^{1/q})\). This function is locally bounded
because \(a(q)\) and \(m_q\) are positive and continuous and \(M(q)\) is
locally bounded. Combining the two estimates in the triangle-inequality
display proves the result for unit vectors, and homogeneity gives the stated
result.
\end{proof}

\subsection{Proof of Theorem~\ref{thm:lp-clean}}
\label{sec:lp-benign}

We prove Theorem~\ref{thm:lp-clean} by specializing
Theorem~\ref{thm:main} to minimum-\(\ell_p\)-norm interpolation. Throughout
this section, \(q\) denotes the dual exponent from
Theorem~\ref{thm:lp-clean}, and \(\|\cdot\|=\|\cdot\|_p\).

\begin{proof}
The assumptions of Setting~\ref{setting:lp-product} hold with
\(\zeta_{i,j}=x_{i,j}/s_N\): these variables are independent, mean-zero,
variance one, and satisfy \(\|\zeta_{i,j}\|_{\psi_2}\le R_X\).

To establish the claimed monotonicity, we fix \(p_0>1\) and prove the result uniformly for all \(p\ge p_0\). Then the dual exponent satisfies
\(1\le q\le p_0/(p_0-1)\). Throughout the proof, \(C\ge1\) and \(c>0\)
may change from line to line, with \(C\) and \(c^{-1}\) having the polynomial dependence stated in
the theorem.

Denote the quantity in square brackets in the theorem by
\[
    \eta:=
    \frac{s_N|N|^{1/q}|S|^{\max\{1/p-1/2,0\}}}{\sqrt n}
    +\left(\frac{r_S+\log(n/\delta)}n\right)^{1/4}
    +\sqrt{\frac{n\log(n|N|/\delta)}{|N|}}
    +\frac{n\log(n|N|/\delta)}{|N|^{\min\{2/q,1\}}}.
\]
After adjusting the constants in the statement, it suffices to prove
\(\Rcal(\widehat\bw)\le C\eta\) under the assumption \(\eta\le c\).

We first verify the sample-size condition of Theorem~\ref{thm:main}. The
second term in \(\eta\) gives
\[
    n\ge C\bigl(r_S+\log(n/\delta)\bigr).
\]
The last term similarly gives
\[
    |N|
    \ge C\bigl(n\log(n|N|/\delta)\bigr)^{\max\{1,q/2\}}.
\]
This lower bound also implies
\[
    \frac{\bigl(n\log(n|N|/\delta)\bigr)^{q/2}}{|N|}
    \le
    \frac{n\log(n|N|/\delta)}{|N|^{\min\{2/q,1\}}},
    \qquad
    |N|^{-1/(2q)}
    \le
    \left(\frac{r_S+\log(n/\delta)}n\right)^{1/4}.
\]
For the first inequality, distinguish \(q\le2\) and \(q\ge2\) and use
the preceding lower bound on \(|N|\); the second follows from the same
bound and \(r_S\ge1\).

We next verify Assumption~\ref{assump:K}. For the \(\ell_p\) norm, the
dual norm associated with \(\bw_N\mapsto\|(\zero,\bw_N)\|_p\) is
\(\|\cdot\|_q\). For \(G\sim\mathcal N(0,1)\), take
\[
    \kappa_N:=s_N\bigl(|N|\E|G|^q\bigr)^{1/q}.
\]
Lemma~\ref{lem:K_subgaussian} verifies the assumption with
\[
    A=\frac{a}{R_X^4\sqrt q},
    \qquad
    \varphi(t)
    =CR_X|N|^{-1/(2q)}+D(q)R_X^{q+4-2/q}t,
\]
where \(D(q)\) depends only on \(q\) and is bounded uniformly for
\(1\le q\le p_0/(p_0-1)\), and \(a>0\) is an absolute constant. In
particular, \(A^{-1}\le CR_X^4\). The
proof of Theorem~\ref{thm:main} uses \(A\) only through \(A^{-2}\), so
its constants have the polynomial dependence asserted in the theorem.

We now verify Assumption~\ref{assump:xnl_concentrate}. Choose
\[
    \tau:=K
    \left(
        \sqrt{\frac{n\log(n|N|/\delta)}{|N|}}
        +\frac{\bigl(n\log(n|N|/\delta)\bigr)^{q/2}}{|N|}
    \right),
\]
where \(K\) is a sufficiently large polynomial in \(R_X\), with degree
and coefficients depending only on \(p_0\). The preceding comparisons
give \(\tau\le K\eta\), so decreasing \(c\) ensures \(\tau\in(0,1)\). In
Proposition~\ref{prop:xnl_conc_subgaussian}, with failure parameter
\(\delta/12\), we have \(L=\sqrt n+\sqrt{\log(24|N|/\delta)}\le
C\sqrt{n\log(n|N|/\delta)}\). Moreover, the constant \(c(q)\) in that
proposition is bounded below uniformly for
\(1\le q\le p_0/(p_0-1)\). The logarithmic term in its probability bound
is at most \(Cn\log(n|N|/\delta)\), while
\[
    \min\left\{|N|\tau^2,(|N|\tau)^{2/q}\right\}
    \ge K^{2/q}n\log(n|N|/\delta).
\]
Since \(R_X\ge1\) and \(q\le p_0/(p_0-1)\), we may choose \(K\) so that
the second term in the probability bound of
Proposition~\ref{prop:xnl_conc_subgaussian} is at most \(\delta/12\).
Thus Assumption~\ref{assump:xnl_concentrate} holds at failure level
\(\delta/6\), with
\[
    \varepsilon_{\mathrm{unif}}(\delta/6)
    \le C\left[
        \sqrt{\frac{n\log(n|N|/\delta)}{|N|}}
        +\frac{\bigl(n\log(n|N|/\delta)\bigr)^{q/2}}{|N|}
        +|N|^{-1/(2q)}
    \right]
    \le C\eta.
\]

It remains to verify the regularity and radius conditions of
Theorem~\ref{thm:main}. Since \(\bw_N^\star=0\), we have \(\nu_N=0\), and
\(\|\bb_S\|_2=b_S^+\sqrt{r_S}\). The definition in
\eqref{eq:main-scales} therefore gives
\[
    \rho
    \le C\sqrt{\frac{r_S+\log(n/\delta)}n}.
\]
Using the displayed formula for \(\varphi\), the uniform bound on \(D(q)\),
and \(R_X\ge1\), we obtain
\[
    \varphi(C'\sqrt\rho)
    \le C\left[
        |N|^{-1/(2q)}
        +\left(\frac{r_S+\log(n/\delta)}n\right)^{1/4}
    \right]
    \le C\eta.
\]
The definition of \(\varepsilon_{\mathrm{flat}}\) now yields
\[
    \varepsilon_{\mathrm{flat}}\le C\eta.
\]
This controls the first term in \eqref{eq:main-regularity}. To control the
second, note that \((\E|G|^q)^{1/q}\) is bounded above and below by
positive constants uniformly for \(1\le q\le p_0/(p_0-1)\). Therefore
\[
    c\frac{s_N|N|^{1/q}}{\sqrt n}
    \le \gamma_N
    \le C\frac{s_N|N|^{1/q}}{\sqrt n},
    \qquad
    \beta_S=|S|^{\max\{1/p-1/2,0\}},
\]
and hence \(\gamma_N\beta_S/b_S^-\le C\eta\). Finally, the third term in
\eqref{eq:main-regularity} satisfies
\[
    \left(
        \frac{R_Xb_S^++R_\xi\sigma}{\min\{b_S^-,\sigma\}}
    \right)^2
    \sqrt{\frac{r_S+\log(n/\delta)}n}
    \le C\eta.
\]
This verifies \eqref{eq:main-regularity} when \(\eta\le c\).

It remains only to check the radius condition. Since \(\beta_S\ge1\), the
bound \(\gamma_N\beta_S\le C\eta\) also controls \(\gamma_N\). Combining
this with \(\eta\le c\), all the terms involving \(\gamma_N\) in
\eqref{eq:main-signal-radius} are at most \(C\eta\). The bound
\(\varepsilon_{\mathrm{flat}}\le C\eta\) controls the first term. Finally,
since \(\eta\le c\le1\),
\[
    \sqrt{\frac{r_S+\log(n/\delta)}n}
    \le
    \left(\frac{r_S+\log(n/\delta)}n\right)^{1/4}
    \le\eta
\]
controls the remaining term. Thus
\[
    r_\delta^2\le C\eta.
\]
Decreasing \(c\) once more ensures that \(r_\delta\le1\). We have now
verified all the hypotheses of Theorem~\ref{thm:main}.

We may therefore apply Theorem~\ref{thm:main}. For the \(\ell_p\) norm,
\[
\begin{aligned}
    \beta_{N,\Sigma}
    &=
    s_N\sup_{\|\bv\|_p=1}\|\bv\|_2
    =
    s_N|N|^{\max\{1/2-1/p,0\}}
    =
    s_N|N|^{\max\{1/q-1/2,0\}}.
\end{aligned}
\]
The regularity bound and \(r_\delta\le1\) imply
\(\beta_Sr_\delta\le C/\gamma_N\). Consequently, the risk bound in
Theorem~\ref{thm:main} gives
\[
    \Rcal(\widehat\bw)
    \le
    Cr_\delta^2
    +
    C s_N^2 |N|^{2\max\{1/q-1/2,0\}}
    \left(
        \frac{1}{\gamma_N}+1
    \right)^2 .
\]
Using \((a+1)^2\le2a^2+2\) and the definition of \(\gamma_N\), the
\(1/\gamma_N^2\) contribution equals
\[
    s_N^2 |N|^{2\max\{1/q-1/2,0\}}\frac1{\gamma_N^2}
    =
    (\E|G|^q)^{-2/q}
    \frac{n}{|N|^{\min\{2/q,1\}}}.
\]
The Gaussian factor is bounded uniformly for
\(1\le q\le p_0/(p_0-1)\). The contribution without
\(1/\gamma_N^2\) is no larger than a constant multiple of the same rate.
Indeed,
\[
    2\max\{1/q-1/2,0\} + \min\{2/q,1\}=2/q
\]
and \(\beta_S\ge1\) give
\[
    \frac{s_N^2|N|^{2\max\{1/q-1/2,0\}}}
    {n/|N|^{\min\{2/q,1\}}}
    \le
    \frac{s_N^2|N|^{2/q}\beta_S^2}{n}
    \le C\eta^2.
\]
Therefore
\[
    \Rcal(\widehat\bw)
    \le
    C
    \left[
        \frac{n}{|N|^{\min\{2/q,1\}}}
        +
        r_\delta^2
    \right].
\]
Finally, \(r_\delta^2\le C\eta\), and
\[
    \frac{n}{|N|^{\min\{2/q,1\}}}
    \le
    \frac{n\log(n|N|/\delta)}{|N|^{\min\{2/q,1\}}}
    \le\eta.
\]
Thus \(\Rcal(\widehat\bw)\le C\eta\), as claimed.

Finally, choosing the constants uniformly over dual exponents in \([1,q]\), a range that shrinks as \(p\) increases, makes \(C\) non-increasing and the small tolerance \(c\) non-decreasing in \(p\). Increasing \(C\) if necessary so that \(C\ge c^{-1}\) preserves the claimed monotonicity and ensures that \(\eta\le C^{-1}\) implies \(\eta\le c\).
\end{proof}

\section{Proof of Theorem~\ref{thm:l1-mean-zero-failure} -- The \texorpdfstring{\(\ell_1\)}{l1} Failure Example}
\label{sec:l1-failure-proof}

\begin{proof}
Let $S=\{1\}, N=\{2,\ldots, d\}$, and recall that the inputs are given by
\[
    x_S\sim\operatorname{Unif}\{\pm1\},
    \qquad
    \text{and}
    \qquad
    x_j\sim\operatorname{Unif}\{\pm s_N\}
    \quad \forall ~ j\in N.
\]
Let the label noise be independent copies of \(\xi:=\sigma(Z_n/\rho_n)\varepsilon\), where 
\(\varepsilon\sim\operatorname{Unif}\{\pm1\}\), $Z_n$, $\rho_n$ are defined as
\[
    Z_n = \begin{cases}
        1 + \sqrt{\log(n)} & \text{w.p. } \frac{1}{n} \\ 
        1 & \text{w.p. } 1 - \frac{1}{n}
    \end{cases}, 
    \qquad
    \rho_n:=\E[Z_n^2]^{1/2},
\]
and $Z_n$ and $\varepsilon$ are independent.

We first verify Assumptions~\ref{assump:subgaussian-input-distribution}
 and~\ref{assump:noise-variables}. By symmetry,
\(\E\xi=0\), and by the definition of \(\rho_n\), \(\E\xi^2=\sigma^2\). Moreover,
\[
\begin{aligned}
    \rho_n^2 = \E Z_n^2
    &=\frac{\log n + 2\sqrt{\log(n)} + 1}{n}+1-\frac1n 
    = 1 + \frac{\log n + 2\sqrt{\log(n)}}{n}.
\end{aligned}
\]
It is easy to check that $\rho_n \in [1, 3/2]$. Moreover, since
$(1+\sqrt{a})^2\le 2(1+a)$ for every $a\ge0$, we have
\[
\begin{aligned}
    \E\exp\left(\frac{\xi^2}{16\sigma^2}\right)
    &\le
    \left(1-\frac1n\right)e^{1/16}
    +
    \frac1n e^{(1+\log n)/8}
    \le e^{1/16}+e^{1/8}2^{-7/8}
    \le2
\end{aligned}
\]
for all $n\geq 2$. Hence \(\|\xi\|_{\psi_2}\le4\sigma\). Moreover, all coordinates of $\bx$ are scaled Rademacher random variables, so
Assumption~\ref{assump:subgaussian-input-distribution} also holds with a
universal constant. Write
\[
    X_N=s_NR,
    \qquad
    R\in\{\pm1\}^{n\times |N|},
\]
and let \(\Ecal_{\mathrm{cube}}\) be the event that every vector in
\(\{\pm1\}^n\) occurs as a column of \(R\). A union bound gives
\[
\begin{aligned}
    \Pr(\Ecal_{\mathrm{cube}}^c)
    &\le
    2^n\left(1-2^{-n}\right)^{|N|}
    \le
    \exp\left(n\log2-\frac{|N|}{2^n}\right).
\end{aligned}
\]
Thus, for \(|N|\ge 2 n2^n\),
\[
    \Pr(\Ecal_{\mathrm{cube}}^c)
    \le e^{-(2-\log 2)n}
    \le e^{-n}.
\]

Work on \(\Ecal_{\mathrm{cube}}\). Since the columns of \(R\) contain all
sign vectors, \(X_N\) has full row rank. Hence, for every
\(\bz\in\R^n\), \(\ell_1\)-duality (see Lemma~\ref{lem:dual}) gives
\begin{equation}
\label{eq:l1-dual-cost}
    \min_{\bw_N:\,X_N\bw_N=\bz}\|\bw_N\|_1
    =
    \max_{\blambda\in\R^n}
    \left\{
        \bz^\top\blambda:
        \|X_N^\top\blambda\|_\infty\le1
    \right\}.
\end{equation}
Moreover, for every \(\blambda\in\R^n\), since $R$ contains all sign vectors,
\begin{equation}
\label{eq:l1-cube-dual-norm}
    \|X_N^\top\blambda\|_\infty
    =
    s_N\max_{\bu\in\{\pm1\}^n}|\bu^\top\blambda|
    =
    s_N\|\blambda\|_1.
\end{equation}
Combining \eqref{eq:l1-dual-cost} and
\eqref{eq:l1-cube-dual-norm}, and using the duality between
\(\ell_1\) and \(\ell_\infty\), yields
\begin{equation}
\label{eq:l1-cost-mean-zero}
    \min_{\bw_N:\,X_N\bw_N=\bz}\|\bw_N\|_1
    =
    \max_{\|\blambda\|_1\le s_N^{-1}}
        \bz^\top\blambda
    =
    \frac1{s_N}\|\bz\|_\infty.
\end{equation}
Since \(\bw^\star=0\), the min-norm interpolation problem is
\begin{align*}
    \min_{\bw:\,X_Sw_S+X_N\bw_N=\bxi}\|\bw\|_1
    =
    \min_{w_S \in \R} 
    \min_{\bw_N: X_N\bw_N=\bxi - X_Sw_S}
        \left(
            |w_S|+\|\bw_N\|_1
        \right),
\end{align*}
so that by \eqref{eq:l1-cost-mean-zero}, the min-norm interpolator must satisfy
\begin{align*}
    \widehat w_S \in \argmin{w_S \in \R}
        \left(
            |w_S|+\frac1{s_N}\|\bxi-X_S w_S\|_\infty
        \right).
\end{align*}
Let
\begin{equation}
\label{eq:l1-profiled-objective}
    H(w)
    :=
    |w|+\frac1{s_N}\|\bxi-X_Sw\|_\infty,
\end{equation}
and define
\[
    \eta_i:=x_{i,S}\xi_i,
    \qquad
    \eta_-:=\min_{i\in[n]}\eta_i,
    \qquad
    \eta_+:=\max_{i\in[n]}\eta_i,
    \qquad
    M_\eta:=\frac{\eta_++\eta_-}{2}.
\]
Since \(x_{i,S}\in\{\pm1\}\), $|\xi_i-x_{i,S}w|=|x_{i, S}\xi_i -w|$, and therefore, 
\begin{align}\label{eq:l1-midrange-identity}
\|\bxi-X_Sw\|_\infty 
= & \max_{i\in[n]}|\eta_i-w| =_{(*)} \max\{|\eta_+-w|,|\eta_--w|\} \nonumber\\
=&  \max\left\{\left|\frac{\eta_+ - \eta_-}{2} + M_\eta-w \right|,
               \left|-\frac{\eta_+ - \eta_-}{2} + M_\eta-w \right|
        \right\} \nonumber\\
=&_{(**)} \frac{\eta_+ - \eta_-}{2} + |M_\eta-w|,
\end{align}
where (*) uses that the maximization problem is one-dimensional 
and (**) uses the inequality $\max\{|a+b|,|-a+b|\} = a + |b|$ for $a\ge 0$ and $b\in\R$.
By \eqref{eq:l1-profiled-objective} and
\eqref{eq:l1-midrange-identity},
\[
\begin{aligned}
    H(w)-H(M_\eta)
    &=
    |w|-|M_\eta|+\frac1{s_N}|w-M_\eta|\\
    &\ge
    \left(\frac1{s_N}-1\right)|w-M_\eta|.
\end{aligned}
\]
Because \(s_N<1\), the right-hand side is strictly positive unless
\(w=M_\eta\). Thus \(M_\eta\) is the unique minimizer of \(H\), and every
minimum-\(\ell_1\)-norm interpolator satisfies
\begin{equation}
\label{eq:l1-signal-midrange}
    \widehat w_S
    =
    M_\eta
    =
    \frac{
        \max_{i\in[n]}(x_{i,S}\xi_i)
        +
        \min_{i\in[n]}(x_{i,S}\xi_i)
    }{2}.
\end{equation}

Let \(\Ecal_{\mathrm{out}}\) be the event that exactly one of
\(\xi_1,\ldots,\xi_n\) has magnitude
\(\sigma(1 + \sqrt{\log n})/\rho_n\). Since each label-noise coordinate has this magnitude with
probability \(1/n\), and there are $n$ choices for which coordinate has this magnitude, 
\[
    \Pr(\Ecal_{\mathrm{out}})
    =
    n \cdot \frac{1}{n}\left(1-\frac1n\right)^{n-1} = \left(1-\frac1n\right)^{n-1} \geq \frac
    {1}{e}.
\]
On
\(\Ecal_{\mathrm{out}}\), exactly one \(\eta_i\) has magnitude
\(\sigma(1 + \sqrt{\log n})/\rho_n\), while all the others have magnitude
\(\sigma/\rho_n\). If the exceptional coordinate is positive, then
\(\eta_+=\sigma(1+\sqrt{\log n})/\rho_n\) and
\(\eta_-\in\{-\sigma/\rho_n,\sigma/\rho_n\}\); the analogous statement holds with
the roles reversed if it is negative. Hence
\eqref{eq:l1-signal-midrange} and \(\rho_n<3/2\) give
\begin{equation}
\label{eq:l1-midrange-lower}
    |\widehat w_S|
    \ge
    \frac{\sigma\sqrt{\log n}}{2\rho_n} \geq \frac{\sigma\sqrt{\log(n)}}{3}.
\end{equation}
Finally, since \(\bw^\star=0\), the excess risk is
\[
\begin{aligned}
    \Rcal(\widehat\bw)
    &=
    \widehat w_S^2+s_N^2\|\widehat\bw_N\|_2^2
    \ge
    \widehat w_S^2
    \ge
    \sigma^2\frac{\log n}{9}.
\end{aligned}
\]
Combining the preceding bounds on the probabilities of
\(\Ecal_{\mathrm{cube}}\) and \(\Ecal_{\mathrm{out}}\),
\[
    \Pr\left(\Ecal_{\mathrm{cube}}\cap\Ecal_{\mathrm{out}}\right)
    \ge \frac1e-e^{-n}
    \ge \frac15.
\]
Together with the preceding risk bound, this proves the theorem.
\end{proof}

\end{document}